\input{style/arxiv-ba.cfg}
\documentclass[ba,linksfromyear,preprint]{imsart}
\makeatletter
   \@ifpackageloaded{natbib}{}{\usepackage{natbib}}
\makeatother

\pubyear{2015}
\volume{10}
\issue{2}
\firstpage{247}
\lastpage{274}
\doi{10.1214/14-BA892}

\makeatletter
\setattribute{email}{text}{}

\newtheorem{theorem}{Theorem}
\newtheorem{lemma}{Lemma}
\newtheorem{definition}{Definition}
\newtheorem{proposition}{Proposition}

\def\RB{{\mathbb R}}
\def\NB{{\mathbb N}}
\def\EB{{\mathbb E}}
\def\Var{{\mathbb Var}}

\def\epsi{\mbox{\boldmath$\epsilon$\unboldmath}}
\def\etab{\mbox{\boldmath$\eta$\unboldmath}}
\def\Si{\mbox{\boldmath$\Sigma$\unboldmath}}
\def\argmin{\mathop{\rm argmin}}
\def\liml{\mathop{\lim}\limits}

\def\NoB{\mathrm{NB}}
\def\PG{\mathrm{PG}}

\def\CoP{\mathrm{CP}}
\def\SB{\mathrm{SB}}
\def\Ga{\mathrm{Ga}}
\def\IGa{\mathrm{IG}}
\def\Po{\mathrm{Po}}
\def\rk{\mathrm{rank}}
\makeatother

\begin{document}

\begin{frontmatter}
\title{Compound Poisson Processes, Latent Shrinkage Priors and
Bayesian Nonconvex Penalization}
\runtitle{Bayesian Nonconvex Penalization}

\begin{aug}
\author[a]{\fnms{Zhihua} \snm{Zhang}\ead[label=e1]{zhihua@sjtu.edu.cn}}
\and
\author[b]{\fnms{Jin} \snm{Li}\ead[label=e2]{lijin@sjtu.com}}

\runauthor{Z. Zhang and J. Li}


\address[a]{Department of Computer Science and Engineering,
Shanghai Jiao Tong University, Shanghai, China, \printead{e1}}
\address[b]{Department of Computer Science and Engineering,
Shanghai Jiao Tong University, Shanghai, China, \printead{e2}}


\end{aug}

%
\begin{abstract}
In this paper we discuss Bayesian nonconvex penalization for sparse
learning problems.
We explore a nonparametric formulation for latent shrinkage parameters
using subordinators which are one-dimensional L\'{e}vy processes.
We particularly study a family of continuous compound Poisson
subordinators and a family of discrete
compound Poisson subordinators. We exemplify four specific
subordinators: Gamma, Poisson, negative binomial and squared Bessel
subordinators.
The Laplace exponents of the subordinators are Bernstein functions, so
they can be used as sparsity-inducing nonconvex penalty functions.
We exploit these subordinators in regression problems, yielding a
hierarchical model with multiple regularization parameters.
We devise ECME (Expectation/Conditional Maximization Either) algorithms
to simultaneously estimate regression coefficients and regularization
parameters. The empirical evaluation of simulated data shows that our
approach is feasible and effective in high-dimensional
data analysis.
\end{abstract}

%


\end{frontmatter}


\section{Introduction}
\label{sec:intro}

Variable selection methods based on penalty theory have received great
attention in high-dimensional
data analysis.
A principled approach is due to the lasso of \citet
{TibshiraniLASSO:1996}, which uses the $\ell_1$-norm penalty.
\citet{TibshiraniLASSO:1996} also pointed out that the lasso estimate can
be viewed as the mode of the posterior distribution. Indeed, the $\ell
_1$ penalty can be transformed into the Laplace prior.
Moreover, this prior can be expressed as a Gaussian scale mixture. This
has thus led to Bayesian developments of the lasso and its
variants~\citep
{Figueiredo:2003,Park:2008,Hans:2009,KyungBA:2010,GriffinBrownBA:2010,LiLinBA:2010}.

There has also been work on nonconvex penalization under a parametric
Bayesian framework.
\citet{ZouLi:2008} derived their local linear approximation (LLA)
algorithm by combining the expectation maximization (EM)
algorithm with an inverse Laplace
transform. In particular, they showed that the $\ell_q$ penalty with
$0<q<1$ can be
obtained by mixing the Laplace distribution with a stable density.
Other authors have shown that the prior induced from a penalty, called
the nonconvex LOG penalty and defined in equation (\ref{eqn:logp}) below,
has an interpretation as a scale mixture of Laplace distributions with an
inverse Gamma mixing distribution~\citep
{CevherNIPS:2009,GarriguesfNIPS:2010,LeeCaronDoucetHolmes:2010,ArmaganDunsonLee}.
Recently, \citet{ZhangEPGIG:2012} extended this class of Laplace
variance mixtures by using a generalized inverse Gaussian mixing distribution.
Related methods include the Bayesian hyper-lasso~\citep{GriffinBrownBHL:2011},
the horseshoe model~\citep{CarvalhoPolson:2010,CarvalhoPS:2009} and
the Dirichlet Laplace prior~\citep{Bhattacharya:2012}.

In parallel,
nonparametric Bayesian approaches have been applied to variable
selection~\citep{GhahramanBS:2006}.
For example, in the infinite Gamma Poisson model~\citep
{TitsiasNIPS:2007} negative binomial processes are used to describe
non-negative integer valued matrices,
yielding a nonparametric Bayesian feature selection approach under an
unsupervised learning setting.
The beta-Bernoulli process provides a nonparametric Bayesian tool
in sparsity modeling~\citep
{ThibauxJordan:2007,BJP:2012,Paisley2009:icml,TehNIPS:2009}.
Additionally, \citet{CaronDoucet:icml} proposed a nonparametric
approach for normal variance mixtures
and showed that the approach is closely related to L\'{e}vy processes.
Later on, \citet{PolsonScott:2011}
constructed sparse priors using increments of subordinators, which
embeds finite dimensional normal variance mixtures in infinite ones.
Thus, this provides a new framework for the construction of
sparsity-inducing priors. Specifically, \citet{PolsonScott:2011}
discussed the use of $\alpha$-stable subordinators and inverted-beta
subordinators for
modeling joint priors of regression coefficients.
\citet{ZhangNIPS:2012} established the connection of two nonconvex
penalty functions, which are referred to as LOG
and EXP and defined in equations (\ref{eqn:logp}) and (\ref{eqn:exp})
below, with the Laplace transforms of the Gamma and Poisson subordinators.
A subordinator
is a one-dimensional L\'{e}vy process that is almost surely
non-decreasing~\citep{Sato:1999}.

In this paper we further study the application of subordinators in
Bayesian nonconvex penalization problems under
supervised learning scenarios. Differing from the previous treatments,
we model latent shrinkage parameters
using subordinators which are defined as stochastic processes of
regularization parameters.
In particular,
we consider two families of compound Poisson subordinators: continuous
compound Poisson subordinators
based on a Gamma random variable~\citep{Aalen:1992} and discrete
compound Poisson subordinators based on a logarithmic random
variable~\citep{Sato:1999}. The corresponding L\'{e}vy measures are
generalized Gamma~\citep{Brix:1999}
and Poisson measures, respectively.
We show that both the Gamma and Poisson subordinators
are limiting cases of these two families of the compound Poisson subordinators.

Since the Laplace exponent of a subordinator is a Bernstein function,
we have two families of nonconvex penalty functions,
whose limiting cases are the nonconvex LOG and EXP.
Additionally,
these two families of nonconvex penalty functions can be defined via
composition of LOG and EXP, while
the continuous and discrete compound Poisson subordinators
are mixtures of Gamma and Poisson processes.

Recall that the latent shrinkage parameter is a stochastic process of
the regularization parameter.
We formulate a hierarchical model with multiple regularization parameters,
giving rise to a Bayesian approach for nonconvex penalization. To
reduce computational expenses, we devise
an ECME (for ``Expectation/Conditional Maximization Either") algorithm
\citep{LiuBio:1994}
which can adaptively adjust the local regularization parameters in
finding the sparse solution simultaneously.

The remainder of the paper is organized as follows.
Section~\ref{sec:levy}
reviews the use of L\'{e}vy processes in Bayesian sparse learning problems.
In Section~\ref{sec:gps} we study two families of compound Poisson processes.
In Section~\ref{sec:blrm} we apply the L\'{e}vy processes to Bayesian
linear regression
and devise an ECME algorithm for finding the sparse solution.
We conduct empirical evaluations using simulated data in Section~\ref
{sec:experiment}, and conclude our work in
Section~\ref{sec:conclusion}.

\section{Problem Formulation}
\label{sec:levy}

Our work is based on the notion of Bernstein and completely monotone
functions as well as subordinators.

\begin{definition} Let $g \in C^{\infty}(0, \infty)$ with $g\geq0$.
The function $g$ is said to be completely monotone if $(-1)^n g^{(n)}
\geq0$ for all $n \in\NB$ and
Bernstein if $(-1)^n g^{(n)} \leq0$ for all $n \in\NB$.
\end{definition}

Roughly speaking,
a \emph{subordinator} is a one-dimensional L\'{e}vy process that is
non-decreasing almost surely. Our work is mainly motivated by
the property of subordinators given in Lemma~\ref{lem:subord}~\citep
{Sato:1999,Applebaum:2004}.

\begin{lemma} \label{lem:subord} If $T=\{T(t): t\geq0\}$ is a
subordinator, then the Laplace transform of its density takes the form
\[
\EB\big(e^{- s T(t)} \big) = \int_{0}^{\infty} { e^{-s \eta}
f_{T(t)}(\eta) d \eta}
\triangleq e^{- t \Psi(s)} \quad\mbox{ for } s> 0,
\]
where $f_{T(t)}$ is the density of $T(t)$ and
$\Psi$, defined on $(0, \infty)$, is referred to as the \emph
{Laplace exponent} of the subordinator and has the following representation
\begin{equation} \label{eqn:psi}
\Psi( s) = \beta s + \int_{0}^{\infty} \big[1- e^{- s u} \big] \nu
(d u).
\end{equation}
Here $\beta\geq0$ and $\nu$ is the L\'{e}vy measure such that $\int
_{0}^{\infty} { \min(u, 1) \nu(d u) } < \infty$. 

Conversely, if $\Psi$ is an arbitrary mapping from $(0, \infty)
\rightarrow(0, \infty)$ given by expression (\ref{eqn:psi}),
then $e^{- t \Psi(s)}$ is
the Laplace transform of the density of a subordinator.
\end{lemma}

It is well known that the Laplace exponent $\Psi$ is Bernstein
and the corresponding Laplace transform $\exp(-t \Psi(s))$ is
completely monotone for any $t\geq0$~\citep{Applebaum:2004}.
Moreover, any function $g: (0, \infty)\to\RB$, with $g(0)=0$,
is a Bernstein function if and only if it has the representation
as in expression (\ref{eqn:psi}).
Clearly,
$\Psi$ as defined in expression (\ref{eqn:psi}) satisfies $\Psi(0)=0$.
As a result,
$\Psi$
is nonnegative, nondecreasing and concave on $(0, \infty)$.

\subsection{Subordinators for Nonconvex Penalty Functions}

We are given a set of training
data $\{({\bf x}_i, y_i): i=1,\ldots, n\}$, where
the ${\bf x}_i \in\RB^{p}$ are the input vectors and the $y_i$ are
the corresponding
outputs.
We now discuss the following\vadjust{\goodbreak} linear regression model:
\[
{\bf y}= {\bf X}{\bf b}+ \epsi,
\]
where ${\bf y}=(y_1, \ldots, y_n)^T$,
${\bf X}=[{\bf x}_1, \ldots, {\bf x}_n]^T$, and $\epsi$ is a Gaussian
error vector $N({\bf0}, \sigma{\bf I}_n)$.
We aim at
finding a sparse estimate of the vector of regression coefficients
${\bf b}
=( b_1, \ldots,\break b_p)^T$ by using a Bayesian nonconvex approach.

In particular,
we consider the following hierarchical model for the regression
coefficients $b_j$'s:
\begin{align*}
p(b_j | \eta_j, \sigma) \; & {\varpropto} \; \exp(-\eta_j
|b_j|/\sigma), \\
[\eta_j] & \stackrel{iid}{\sim} p(\eta_j), \\
\sigma& \sim\IGa(\alpha_{\sigma}/2, \beta_{\sigma}/2),
\end{align*}
where the $\eta_j$'s are referred to as latent shrinkage parameters,
and the inverse Gamma prior has the following parametrization:
\[
\IGa(\alpha_{\sigma}/2, \beta_{\sigma}/2) = \frac{(\beta_{\sigma
}/2)^{\alpha_{\sigma}/2}}{\Gamma(\frac{\alpha_{\sigma}}{2})}
\sigma^{-(\frac{\alpha_{\sigma}}{2}{+}1)} \exp\Big({-}\frac
{\beta_{\sigma}}{2\sigma} \Big).
\]
Furthermore, we regard $\eta_j$ as $T(t_j)$, that is, $\eta
_j=T(t_j)$. Here $\{T(t): t\geq0\}$ is defined as a subordinator.

Let $\Psi(s)$, defined on $(0, \infty)$, be the Laplace exponent of
the subordinator.
Taking $s=|b|$, it can be shown that $\Psi(|b|)$ defines a nonconvex
penalty function of $b$ on $(-\infty, \infty)$.
Moreover, $\Psi(|b|)$ is nondifferentiable at the origin because $\Psi
'(0^{+})>0$ and $\Psi'(0^{-})<0$.
Thus, it is able to induce sparsity. In this regard, $\exp(- t \Psi
(|b|))$ forms a prior for $b$.
From Lemma~\ref{lem:subord} it follows that the prior can be defined
via the Laplace transform. In summary, we have the following theorem.

\begin{theorem} \label{thm:lapexp00}
Let $\Psi$ be a nonzero Bernstein function on $(0, \infty)$. If
$\liml_{s \to0+} \Psi(s)=0$, then $\Psi(|b|)$
is a nondifferentiable and nonconvex function of $b$ on $(-\infty,
\infty)$. Furthermore,
\[
\exp(- t \Psi(|b|)) = \int_{0}^{\infty}{ \exp(- |b| \eta)
f_{T(t)} (\eta) d \eta}, \; t\geq0,
\]
where $\{T(t): t\geq0\}$ is some subordinator.
\end{theorem}

Recall that $T(t)$ is defined as the latent shrinkage parameter $\eta$
and in Section~\ref{sec:blrm} we will see that $t$ plays the same role as
the regularization parameter (or tuning hyperparameter). Thus,
there is an important connection
between the latent shrinkage parameter and the corresponding
regularization parameter; that is, $\eta=T(t)$. Because $\eta_j=T(t_j)$,
each latent shrinkage parameter $\eta_j$ corresponds to a
local regularization parameter $t_j$. Therefore we have a nonparametric Bayesian
formulation for the latent shrinkage parameters $\eta_j$'s.

It is also worth pointing out that
\[
{\exp(- t \Psi(|b|)) }
= 2 \int_{0}^{\infty}{L(b|0, (2\eta)^{-1}) \eta^{-1} f_{T(t)}(\eta
) d \eta},
\]
where $L(b|u, \eta)$ denotes a Laplace distribution with density given by
\[
p(b|u, \eta) = \frac{1}{4 \eta} \exp\Big(- \frac{1}{2 \eta}
|b-u|\Big).
\]
Thus, if $0<\int_{0}^{\infty}{ \eta^{-1} f_{T(t)}(\eta) d \eta} =
M < \infty$, then
$f_{T^{*}(t)} \triangleq\eta^{-1} f_{T(t)}(\eta)/M$ defines the
proper density of some random variable (denoted $T^{*}(t)$).
Subsequently, we obtain a proper prior $\exp(- t \Psi(|b|))/M$ for
$b$. Moreover,
this prior can be regarded as a Laplace scale mixture, i.e.,
the mixture of $L(b|0, (2\eta)^{-1})$ with mixing distribution
$f_{T^{*}(t)}(\eta)$.
If $\int_{0}^{\infty}{\eta^{-1} f_{T(t)}(\eta) d \eta} = \infty$,
then $f_{T^{*}(t)}$ is not a proper density.
Thus, $\exp(- t \Psi(|b|))$ is also improper as a prior of $b$.
However, we still treat
$\exp(- t \Psi(|b|))$ as the mixture of $L(b|0, (2\eta)^{-1})$ with
mixing distribution $f_{T^{*}(t)}(\eta)$.
In this case, we employ the terminology
of pseudo-priors for the density, which is also used by~\citet
{PolsonScottSVM:2011}.

\subsection{The Gamma and Poisson Subordinators}

Obviously, $\Psi(s)=s$ is Bernstein.
It is an extreme case,
because we have that $\beta=1$, $\nu(d u) =\delta_{0}(u) d u$ and
that $ f_{T(t)}(\cdot) =\delta_{t}(\cdot)$, where $\delta_{t}(\cdot)$
denotes the Dirac Delta measure at $t$, which corresponds to the
deterministic process $T(t)=t$.
We can exclude this case by assuming $\beta=0$ in expression (\ref
{eqn:psi}) to obtain a strictly concave Bernstein function.
In fact,
we can impose the condition $\liml_{s\rightarrow\infty} \frac{\Psi
(s)}{s} =0$. This in turn
leads to $\beta=0$
due to $\liml_{s\rightarrow\infty} \frac{\Psi(s)}{s} = \beta$.
In this paper we exploit Laplace exponents in nonconvex penalization
problems. For this purpose,
we will only consider a subordinator without drift, i.e., $\beta=0$.
Equivalently, we always assume that $\liml_{s\rightarrow\infty}
\frac{\Psi(s)}{s} =0$.

We here take the nonconvex LOG and EXP
penalties as two concrete examples~\citep[also see][]{ZhangNIPS:2012}.
The LOG penalty is defined by
\begin{equation} \label{eqn:logp}
\Psi(s) = \frac{1} {\xi} \log\big({\gamma} {s} {+}1 \big), \quad
\gamma, \; \xi> 0,
\end{equation}
while the EXP penalty is given by
\begin{equation} \label{eqn:exp}
\Psi(s) = \frac{1} {\xi} (1- \exp(- \gamma s )), \quad\gamma, \;
\xi> 0.
\end{equation}
Clearly, these two functions are Bernstein on $(0, \infty)$. Moreover,
they satisfy $\Psi(0)=0$ and
$\liml_{s\rightarrow\infty} \frac{\Psi(s)}{s}=\liml_{s\to\infty
} \Psi'(s)=0$.
It is also directly verified that
\[
\frac{1} {\xi} \log\big( {\gamma} s {+}1 \Big) = \int
_{0}^{\infty} {\big[1- \exp(- s u)\big] \nu(du) },
\]
where the L\'{e}vy measure $\nu$ is given by
\[
\nu(du) = \frac{1}{\xi u} \exp(- u/\gamma) d u.
\]
The corresponding subordinator $\{T(t):t\geq0\}$ is
a Gamma subordinator, because each $T(t)$ follows a Gamma distribution
with parameters $(t/\xi, \gamma)$,
with density given by
\[
f_{T(t)} (\eta) = \frac{\gamma^{-\frac{t}{\xi} } }{\Gamma({t}/
{\xi} )} \eta^{\frac{t}{ \xi} -1} \exp(- \gamma^{-1} \eta)\;
(\mbox{also denoted } \Ga(t/\xi, \gamma)).
\]

We also note that the corresponding pseudo-prior is given by
\[
\exp(- t \Psi(|b|)) = \big( {\gamma} {|b|}{+}1 \big)^{- t/ \xi}
\propto\int_{0}^{\infty}{ L(b|0, \eta^{-1}) \eta^{-1}
f_{T(t)}(\eta) d \eta}.
\]
Furthermore, if $t> \xi$, the pseudo-prior is a proper distribution,
which is the mixture of $L(b|0, \eta^{-1})$
with mixing distribution $\Ga(\eta| \xi^{-1} t {-}1, \gamma)$.

As for the EXP penalty, the L\'{e}vy measure is $\nu(d u) = \xi^{-1}
\delta_{\gamma}(u) d u$.
Since
\[
\int_{\RB} { \big[1{-} \exp({-} \gamma|b|) \big] d b} = \infty,
\]
then ${\xi^{-1}} [1- \exp(- \gamma|b|) ]$ is an improper prior of $b$.
Additionally,
$\{T(t):t\geq0 \}$ is a Poisson subordinator. Specifically, $T(t)$ is
a Poisson distribution with intensity $1/\xi$
taking values on the set $\{k \gamma: k\in\NB\cup\{0 \} \}$. That is,
\begin{equation} \label{eqn:possion}
\Pr(T(t) = k \gamma)= \frac{ (t/\xi)^k}{k!} e^{- t/\xi}, \mbox{
for } k=0,1,2, \ldots
\end{equation}
which we denote by $T(t) \sim\Po(1/\xi)$.

\section{Compound Poisson Subordinators}
\label{sec:gps}

In this section we explore the application of compound Poisson subordinators
in constructing nonconvex penalty functions.
Let $\{Z(k): k \in\NB\}$ be a sequence of independent and identically
distributed (i.i.d.) real valued
random variables with common law $\mu_{Z}$, and let $K \in\NB\cup\{
0 \}$
be a Poisson process with intensity $\lambda$ that is independent of
all the $Z(k)$.
Then $T(t) \triangleq Z(K(1)) + \cdots+ Z(K(t))$, for $t\geq0$,
follows a compound Poisson distribution with density $f_{T(t)}(\eta)$
(denoted $\CoP(\lambda t, \mu_Z)$), and hence $\{T(t): t\geq0\}$ is
called a compound Poisson process.
A compound Poisson process is a subordinator if and only if the $Z(k)$
are nonnegative random variables~\citep{Sato:1999}.
It is worth pointing out that if $\{T(t): t\geq0 \}$ is the Poisson
subordinator given in expression (\ref{eqn:possion}),
it is equivalent to saying that $T(t)$ follows $\CoP({t}/{\xi},
\delta_{\gamma})$.

We particularly study two families of nonnegative random variables $Z(i)$:
nonnegative continuous random variables and nonnegative discrete random
variables. Accordingly, we have continuous and discrete
compound Poisson subordinators $\{T(t): t\geq0 \}$.
We will show that both the Gamma and Poisson subordinators
are limiting cases of the compound Poisson subordinators.

\subsection{Compound Poisson Gamma Subordinators}

In the first family $Z(i)$
is a Gamma random variable. In particular, let $\lambda=\frac{\rho
{+}1}{\rho\xi}$ and the $Z(i)$ be i.i.d.\ from the $\Ga\big(\rho,
\frac{\rho{+}1}{\gamma}\big)$ distribution, where $\rho>0$, $\xi
>0$ and $\gamma>0$.
The compound Poisson subordinator can be written as follows
\[
T(t) = \left\{
\begin{array}{ll} Z(K(1)) + \cdots+ Z(K(t)) & \mbox{ if } K(t)>0, \\
0 & \mbox{ if } K(t)=0.
\end{array}
\right.
\]
The density of the subordinator is then given by
\begin{equation} \label{eqn:first_tt}
f_{T(t)}(\eta) = \exp\Big( {-}\frac{(\rho{+}1) t}{\rho\xi} \Big
) \bigg\{\delta_{0}(\eta) {+} \exp\Big({-}\frac{ (\rho{+}1) \eta
}{\gamma} \Big) \sum_{k=1}^{\infty}
\frac{(\rho{+}1)^{k (\rho{+}1)} (\frac{t}{\xi})^{k} (\frac{\eta
}{\gamma})^{k \rho} } {k! \rho^k \Gamma(k \rho) \eta} \bigg\}.
\end{equation}
We denote it by $\PG(t/\xi, \gamma, \rho)$. The mean and variance are
\[
\EB(T(t)) = \frac{\gamma t}{\xi} \quad\mbox{ and } \quad{\Var
}(T(t))= \frac{\gamma^2 t}{\xi},
\]
respectively.
The Laplace transform is given by
\[
\EB(\exp(-s T(t))) = \exp(- t \Psi_{\rho}(s)),
\]
where $\Psi_{\rho}$ is a Bernstein function of the form
\begin{equation} \label{eqn:first}
\Psi_{\rho}(s) = \frac{\rho{+}1}{\rho\xi}\Big[1- \big(1+ \frac
{\gamma}{\rho{+}1} s \big)^{-\rho} \Big].
\end{equation}
The corresponding L\'{e}vy measure is given by
\begin{equation} \label{eqn:first_nu}
\nu(d u) = \frac{\gamma}{\xi} \frac{((\rho{+}1)/\gamma)^{\rho
{+}1}}{\Gamma(\rho{+}1)} u^{\rho-1} \exp\Big( {-} \frac{\rho
{+}1}{\gamma} u\Big) d u.
\end{equation}
Notice that $\frac{\xi}{\gamma} u \nu(d u)$ is a Gamma measure for
the random variable $u$. Thus, the
L\'{e}vy measure $\nu(d u)$ is referred to as a generalized Gamma
measure~\citep{Brix:1999}.

The Bernstein function $\Psi_{\rho}(s)$ was studied by \citet
{Aalen:1992} for survival analysis.
However, we consider its application in sparsity modeling.
It is clear that $\Psi_{\rho}(s)$ for $\rho>0$ and $\gamma>0$
satisfies the conditions
$\Psi_{\rho}(0)=0$ and $\liml_{s\rightarrow\infty} \frac{\Psi
_{\rho}(s)}{s}=0$. Also, $\Psi_{\rho}(|b|)$ is a nonnegative and nonconvex
function of $b$ on $(-\infty, \infty)$, and it is an increasing function
of $|b|$ on $[0, \infty)$. Moreover, $\Psi_{\rho}(|b|)$ is
continuous w.r.t.\ $b$ but nondifferentiable at the origin. This
implies that $\Psi_{\rho}(|b|)$ can be treated as a sparsity-inducing penalty.

We are interested in the limiting cases that $\rho=0$ and $\rho
=+\infty$.

\begin{proposition} \label{pro:first} Let $\PG(t/\xi, \gamma, \rho
)$, $\Psi_{\rho}(s)$ and $\nu(du)$ be defined by expressions (\ref
{eqn:first_tt}), (\ref{eqn:first})
and (\ref{eqn:first_nu}), respectively. Then
\begin{enumerate}
\item[\emph{(1)}] $\liml_{\rho\to0+} \Psi_{\rho}(s) = \frac
{1}{\xi} \log\big({\gamma} s {+}1 \big)$ and $\liml_{\rho\to
\infty} \Psi_{\rho}(s) = \frac{1}{\xi} (1- \exp(- \gamma s))$;
\item[\emph{(2)}] $\liml_{\rho\to0+} \PG(t/\xi, \gamma, \rho)
= \Ga(t/\xi, \gamma)$ and $\liml_{\rho\to\infty} \PG(t/\xi,
\gamma, \rho) = \CoP(t/\xi, \delta_{\gamma})$;
\item[\emph{(3)}] $\liml_{\rho\to0+} \nu(du)
= \frac{1}{\xi u} \exp(- \frac{u}{\gamma}) d u$ and $\liml_{\rho
\to\infty} \nu(du) = \frac{1}{\xi}\delta_{\gamma}(u) d u$.
\end{enumerate}
\end{proposition}
This proposition can be obtained by using direct algebraic computations.
Proposition~\ref{pro:first} tells us that the limiting cases yield the
nonconvex LOG and EXP functions. Moreover, we see that
$T(t)$ converges in distribution to a Gamma random variable with shape
$t/\xi$ and scale $\gamma$, as $\rho\to0+$, and to a
Poisson random variable
with mean $t/\xi$, as $\rho\to\infty$.

It is well known that $\Psi_{0}$
degenerates to the LOG function~\citep{Aalen:1992,Brix:1999}. Here we
have shown that
$\Psi_{\rho}$ approaches to EXP as $\rho\to\infty$.
We list another special example in Table~\ref{tab:exam} when $\rho
=1$. We refer to the corresponding penalty as a \emph
{linear-fractional} (LFR) function.
For notational simplicity, we respectively replace $\gamma/2$ and $\xi
/2$ by $\gamma$ and $\xi$ in the LFR function.
The density of the subordinator for the LFR function is given by
\[
f_{T(t)}(\eta)= e^{-\frac{t}{\xi}} \Big\{ \delta_{0}({\eta}) +
e^{- \frac{\eta}{\gamma}} \frac{\sqrt{{t}/\xi}
I_1\big(2 \sqrt{ {t}\eta/(\xi\gamma) } \big) }{\gamma\sqrt{\eta
/\gamma}} \Big\}.
\]
We thus say
each $T(t)$ follows a squared Bessel process
without drift \citep{YuanAISM:2000},
which is a mixture of a Dirac delta measure and a randomized Gamma
distribution~\citep{FellerBook:1971}.
We denote the density of $T(t)$ by $\SB({t}/{\xi}, \gamma)$.

\begin{table}[!ht]\setlength{\tabcolsep}{1.0pt}
\begin{small}
\begin{center}
\caption{Bernstein functions LOG, EXP, LFR, and CEL, defined on $[0,
\infty)$, and the corresponding
L\'{e}vy measures and subordinators ($\xi>0$ and $\gamma> 0$).}
\label{tab:exam}
\begin{minipage} {15cm}
\begin{tabular}{lllll}
\hline
& Bernstein Functions & L\'{e}vy Measures $\nu(du)$ & Subordinators
$T(t)$ & Priors \\ \hline
LOG & $\Psi_{0}(s)=\Phi_{0}(s)= \frac{1}{\xi} \log\big({\gamma}
s {+}1 \big)$ & $ \frac{1}{\xi u} \exp( {-} \frac{u}{\gamma}) d u$
& $\Ga(t/\xi, \gamma)$ & Proper$^a$ \\
EXP & $\Psi_{\infty}(s)=\Phi_{\infty}(s)= \frac{1}{\xi} [1 {-}
\exp({-} \gamma s)]$ & $ \frac{1}{\xi} \delta_{\gamma}(u) d u$ &
$\CoP(t/\xi, \delta_{k \gamma})$ & Improper \\
LFR & $\Psi_{1}(s) = \frac{1}{\xi} \frac{\gamma s}{\gamma s {+}1}$
& $ \frac{1}{\xi\gamma} \exp(- \frac{u}{\gamma} ) d u$ & $ \SB
({t}/{\xi}, \gamma)$ & Improper \\
CEL & $\Phi_{1}(s) = \frac{1}{\xi} \log[2 {-}\exp({-}\gamma s)]$ &
$ \frac{1}{\xi}\sum_{k=1}^{\infty} \frac{1}{k 2^k} \delta_{k
\gamma}(u) d u$ & $\NoB({t}/{\xi}, {1}/{2}, \delta_{k \gamma})$ &
Improper \\\hline\\[-9pt]
\multicolumn{5}{@{\quad}l}{$^a$It is proper only when $t>\xi$.}\\
\end{tabular}
\end{minipage}
\end{center}
\end{small}
\end{table}

\subsection{Negative Binomial Subordinators}

In the second case, we consider a family of discrete compound Poisson
subordinators. Particularly,
$Z(i)$ is discrete and takes values on $\{k \alpha: k\in\NB\cup\{0\}
\}$. And it
is defined as logarithmic distribution $\log(1{-}q)$, where $\alpha
\neq0$ and $q\in(0, 1)$, with probability mass function
given by
\[
\Pr(Z(i)= k \alpha) = - \frac{(1-q)^k}{k \log(q)}.\vadjust{\goodbreak}
\]
Moreover, we let $K(t)$ have a Poisson distribution with intensity
$-(\rho{+}1) \log(q)/\xi$, where $\rho>0$.
Then $T(t)$ is distributed according to a negative binomial (NB)
distribution~\citep{Sato:1999}. The probability mass function of
$T(t)$ is given by
\begin{equation} \label{eqn:second_tt}
\Pr(T(t)=k \alpha) = \frac{\Gamma(k{+}(\rho+1) t/\xi)}{k! \Gamma
((\rho+1) t/\xi)} q^{\frac{(\rho+1) t}{\xi}} (1-q)^k,
\end{equation}
which is denoted as $\NoB((\rho{+}1) t/\xi, q, \delta_{k \alpha
})$. We thus say that $T(t)$ follows an NB subordinator.
Let $q=\frac{\rho}{\rho+1}$ and $\alpha=\frac{\rho}{\rho+1}
\gamma$. It can be verified that $\NoB\big((\rho{+}1) t/\xi, \frac
{\rho}{\rho+1}, \delta_{ \frac{k \gamma\rho}{\rho+1}}\big)$ has
the same mean and variance as the $\PG(t/\xi, \gamma, \rho)$ distribution.
The corresponding Laplace transform then
gives rise to a new family of Bernstein functions, which is given by
\begin{equation} \label{eqn:second}
\Phi_{\rho}(s) \triangleq\frac{\rho{+}1}{\xi} \log\Big[\frac{1
{+} \rho}{\rho} - \frac{1}{\rho}
\exp(-\frac{\rho}{\rho{+}1} \gamma s)\Big].
\end{equation}
We refer to this family of Bernstein functions as \emph{compound
EXP-LOG} (CEL) functions.
The first-order derivative of $\Phi_{\rho}(s)$ w.r.t.\ $s$ is given by
\[
\Phi_{\rho}'(s) = \frac{\gamma}{\xi} \frac{\rho\exp(-\frac
{\rho}{\rho{+}1} \gamma s)}{{1{+}\rho} - \exp(-\frac{\rho}{\rho
{+}1} \gamma s)}.
\]

The L\'{e}vy measure for $\Phi_{\rho}(s)$ is given by
\begin{equation} \label{eqn:second_nu}
\nu(d u) = \frac{\rho+1}{\xi} \sum_{k=1}^{\infty} \frac{1}{k
(1{+}\rho)^k} \delta_{ \frac{k \gamma\rho}{\rho{+}1}}(u)du.
\end{equation}
The proof is given in Appendix~1. 
We call this L\'{e}vy measure a \emph{generalized Poisson measure}
relative to the generalized Gamma measure.

Like $\Psi_{\rho}(s)$, $\Phi_{\rho}(s)$ can define a family of
sparsity-inducing nonconvex penalties.
Also, $\Phi_{\rho}(s)$ for $\rho>0$, $\xi>0$ and $\gamma>0$
satisfies the conditions
$\Phi_{\rho}(0)=0$, $\liml_{s\rightarrow\infty} \frac{\Phi_{\rho
}(s)}{s}=0$ and $\liml_{s\to0} \Phi'_{\rho}(s)=\frac{\gamma}{\xi}$.
We present a special CEL function $\Phi_{1}$ as well as the
corresponding $T(t)$ and $\nu(du)$ in Table~\ref{tab:exam}, where we
replace $\xi/2$ and $\gamma/2$ by $\xi$ and $\gamma$ for notational
simplicity.
We now consider the limiting cases.

\begin{proposition} \label{pro:8} Assume
$\nu(du)$ is defined by expression (\ref{eqn:second_nu}) for fixed
$\xi>0$ and $\gamma>0$. Then we have that
\begin{enumerate}
\item[\emph{(a)}] $\liml_{\rho\to\infty} \Phi_{\rho}(s) = \frac
{1}{\xi}(1-\exp(-\gamma s))$ and $\liml_{\rho\to0+} \Phi_{\rho
}(s) = \frac{1}{\xi}\log(1+\gamma s)$.
\item[\emph{(b)}] $\liml_{\rho\to\infty} \Phi'_{\rho}(s) =
\frac{\gamma}{\xi} \exp(- \gamma s)$ and $\liml_{\rho\to0+} \Phi
'_{\rho}(s) = \frac{\gamma}{\xi} \frac{1}{1+\gamma s}$.
\item[\emph{(c)}] $\liml_{\rho\to\infty} \nu(du) = \frac{1}{\xi
}\delta_{\gamma}(u) d u$ and $\liml_{\rho\to0+} \nu(du)
= \frac{1}{\xi u} \exp(- \frac{u}{\gamma}) d u$.\vadjust{\eject}
\item[\emph{(d)}] $\liml_{\rho\to\infty} \NoB\big({(\rho{+}1)
t}/{\xi}, {\rho}/{(\rho{+}1)}, \delta_{{k \rho\gamma}/{(\rho
{+}1)}}\big)
= \CoP({t}/{\xi}, \delta_{\gamma})$ and
\[
\lim_{\rho\to0+} \Pr(T(t)\leq\eta) = \int_{0}^{\eta} {\frac
{\gamma^{-t/\xi} }{\Gamma(t/\xi)} u^{\frac{t}{\xi}-1} \exp
(-\frac{u}{\gamma}) d u}.
\]
\end{enumerate}
\end{proposition}

Notice that
\begin{align*}
\lim_{\rho\to0+} \int_{0}^{\infty}{\exp(- u s) u \nu(d u)} =
\lim_{\rho\to0+} \Phi_{\rho}'(s)
= \frac{\gamma}{\xi} \frac{1}{1 {+} \gamma s}
= \frac{1}{\xi}\int_{0}^{\infty}{\exp\Big( {-} u s {-}\frac
{u}{\gamma} \Big) d u}.
\end{align*}
This shows that $\nu(d u)$ converges to $\frac{1}{\xi}u^{-1} \exp
(-\frac{u}{\gamma})$,
as $\rho\to0$. Analogously, we obtain the second part of
Proposition~\ref{pro:8}-(d), which implies that as $\rho\to0$,
$T(t)$ converges in distribution to a Gamma random variable with shape
parameter $t/\xi$ and scale parameter~$\gamma$.
An alternative proof is given in Appendix~2. 

Proposition~\ref{pro:8} shows that $\Phi_{\rho}(s)$ degenerates to
EXP as $\rho\to\infty$, while to LOG as $\rho\to0$.
This shows an interesting connection between $\Psi_{\rho}(s)$ in
expression (\ref{eqn:first}) and $\Phi_{\rho}(s)$
in expression (\ref{eqn:second}); that is,
they have the same limiting behaviors.

\subsection{Gamma/Poisson Mixture Processes}

We note that for $\rho> 0$,
\[
\Psi_{\rho}(s) = \frac{\rho{+}1}{\rho\xi}\Big[1- \exp\Big
({-}\rho\log(\frac{\gamma s}{\rho{+}1}+1)\Big) \Big]
\]
which is a composition of the LOG and EXP functions,
and that
\[
\Phi_{\rho}(s) = \frac{\rho{+}1}{\xi}\log\Big[1 {+} \frac
{1}{\rho}(1{-}\exp({-}\frac{\rho}{\rho{+}1} \gamma s))\Big]
\]
which is a composition of the EXP and LOG functions. In fact, the
composition of any two Bernstein functions is still Bernstein.
Thus, the composition is also the Laplace exponent of some
subordinator, which is then a mixture of the subordinators
corresponding to
the original two Bernstein functions~\citep{Sato:1999}. This leads us
to an alternative derivation for the subordinators corresponding to
$\Psi_{\rho}$ and $\Phi_{\rho}$. That is, we have the following
theorem whose proof is given in Appendix~3. 

\begin{theorem} \label{thm:poigam}
The subordinator $T(t)$ associated with $\Psi_{\rho}(s)$
is distributed according to the mixture of $\Ga(k\rho, \gamma/(\rho
{+}1))$ distributions with $\Po(k|(\rho{+}1) t/(\rho\xi))$ mixing,
while $T(t)$ associated with $\Phi_{\rho}(s)$ is distributed
according to the mixture of $\CoP(\lambda, \delta_{k \rho\gamma
/{(\rho{+}1})})$ distributions with $\Ga(\lambda|(\rho{+}1)t/\xi,
1/\rho)$ mixing.
\end{theorem}

Additionally, the following theorem illustrates a limiting property of
the subordinators as $\gamma$ approaches 0.\vadjust{\goodbreak}

\begin{theorem} \label{thm:limit} Let $\rho$ be a fixed constant on
$[0, \infty]$.
\begin{enumerate}
\item[\emph{(a)}] If $T(t) \sim\PG({t}/{\xi}, \gamma, \rho)$
where $\xi=\frac{\rho{+}1}{\rho}\Big[1 - (1{+}\frac{\gamma}{\rho
{+}1})^{{-}\rho} \Big]$ or $\xi=\gamma$,
then $T(t)$ converges in probability to $t$, as $\gamma\to0$.
\item[\emph{(b)}] If $T(t) \sim\NoB((\rho{+}1) t/\xi, \rho/(\rho
{+}1), \delta_{k \rho\gamma/(\rho{+}1)})$ where
\[
\xi= {(\rho{+}1)} \log\Big[1 {+} \frac{1}{\rho}(1{-}\exp
({-}\frac{\rho}{\rho{+}1} \gamma))\Big]
\]
or $\xi=\gamma$,
then $T(t)$ converges in probability to $t$, as $\gamma\to0$.
\end{enumerate}
\end{theorem}
The proof is given in Appendix~4. 
Since ``$T(t)$ converges in probability to $t$" implies ``$T(t)$
converges in distribution to $t$," we have that
\[
\liml_{\gamma\to0} \PG(t/\xi, \gamma, \rho)
\overset{d}{=} \delta_t
\mbox{ and }
\liml_{\gamma\to0}\NoB((\rho{+}1) t/\xi, \rho/(\rho{+}1),
\delta_{k \rho\gamma/(\rho{+}1)}) \overset{d}{=} \delta_{t}.
\]

Finally, consider the four nonconvex penalty function given in
Table~\ref{tab:exam}.
We present the following property. That is,
when $\xi=\gamma$ and for any fixed $\gamma>0$, we have
\begin{equation} \label{eqn:relat}
\frac{1}{\gamma} \log[2 {-}\exp({-}\gamma s)] \leq\frac{s}{\gamma
s {+}1} \leq\frac{1}{\gamma} [1 {-} \exp( {-} \gamma s)] \leq\frac
{1}{\gamma} \log\big({\gamma} s {+}1 \big)
\leq s,
\end{equation}
with equality only when $s=0$. The proof is given in Appendix~5. %
This property is also illustrated in Figure~\ref{fig:penalty}.

\section{Bayesian Linear Regression with Latent Subordinators}
\label{sec:blrm}

\begin{figure}[!ht]
\includegraphics{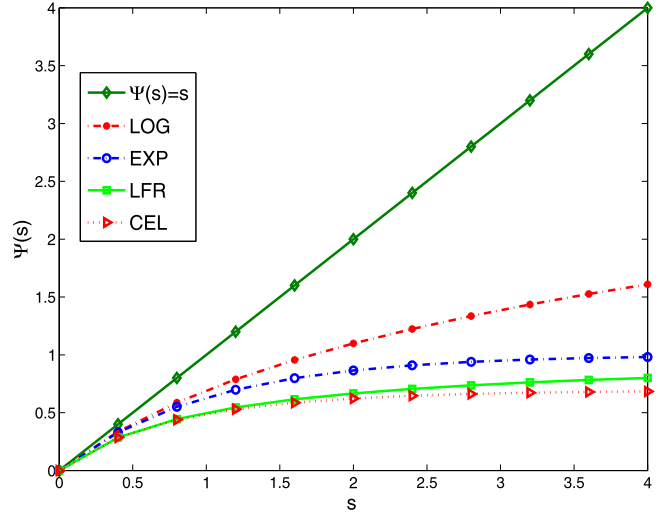}
\caption{The four nonconvex penalties $\Psi(s)$ in Table~\ref
{tab:exam} with $\xi= \gamma=1$ and $\Psi(s)=s$.}
\label{fig:penalty}
\end{figure}

We apply the compound Poisson subordinators to the Bayesian sparse
learning problem given in Section~\ref{sec:levy}.
Defining $T(t)=\eta$,
we rewrite the hierarchical representation for the
joint prior of the $b_j$ under the regression framework. That is, we
assume that
\begin{eqnarray*}
[b_j|\eta_j, \sigma] & \stackrel{ind}{\sim} & L(b_j|0, \sigma
(2\eta_j)^{-1}), \\
f_{T^{*}(t_j)}(\eta_j) & {\propto} & \eta_j^{-1} f_{T(t_j)}(\eta_j),
\end{eqnarray*}
which implies that
\[
p(b_j, \eta_j|t_j, \sigma) \; {\propto} \; \sigma^{-1} \exp\Big
(-\frac{\eta_j}{\sigma} |b_j| \Big) f_{T(t_j)}(\eta_j).
\]
The joint marginal pseudo-prior of the $b_j$'s is given by
\begin{align*}
p^{*}({\bf b}|{\bf t}, \sigma) & = \prod_{j=1}^p p^{*}(b_j|{t_j},
\sigma
) = \prod_{j=1}^p \sigma^{-1} \int_{0}^{\infty}{\exp\Big(-\frac
{\eta_j}{\sigma} |b_j| \Big) f_{T(t_j)}(\eta_j) d \eta_j} \\
& = \prod_{j=1}^p \sigma^{-1} \exp\Big(-t_j \Psi\Big(\frac
{|b_j|} {\sigma} \Big) \Big).
\end{align*}
We will see in Theorem~\ref{thm:poster} that the full conditional
distribution $p({\bf b}| \sigma, {\bf t}, {\bf y})$
is proper.
Thus, the maximum \emph{a posteriori} (MAP) estimate of ${\bf b}$ is based
on the following optimization problem:
\[
\min_{{\bf b}} \; \Big\{ \frac{1}{2} \|{\bf y}- {\bf X}{\bf b}\|_2^2 +
\sigma\sum
_{j=1}^p t_j \Psi(|b_j|/\sigma) \Big\}.
\]
Clearly, the $t_j$'s are local regularization parameters and the $\eta
_j$'s are latent shrinkage parameters. Moreover, it is interesting that
$\{T(t): t\geq0\}$ (or $\eta$) is defined as a subordinator w.r.t.\ $t$.

The full conditional distribution $p(\sigma|{\bf b}, \etab, {\bf y})$
is conjugate w.r.t.\ the prior, which is $\sigma\sim\IGa(\frac
{a_{\sigma}}{2}, \frac{b_{\sigma}}{2})$.
Specifically, it is an inverse Gamma distribution of the form
\[
p(\sigma|{\bf b}, \etab, {\bf y})\varpropto\frac{1}{\sigma^{\frac
{n{+}2p+a_{\sigma}}{2} +1}} \exp\Big[ {-} \frac{b_{\sigma} + \|
{\bf y}
-{\bf X}{\bf b}\|_2^2 + 2 \sum_{j=1}^p \eta_j |b_j| }{2 \sigma} \Big].
\]
In the following experiment, we use an improper prior of the form
$p(\sigma) \varpropto\frac{1}{\sigma}$ (i.e., $a_{\sigma
}=b_{\sigma}=0$).
Clearly, $p(\sigma|{\bf b}, \etab, {\bf y})$ is still an inverse Gamma
distribution in this setting.
Additionally, based on
\[
p({\bf b}| \etab, \sigma, {\bf y}) \varpropto\exp\Big[-\frac{1}{2
\sigma
}\|{\bf y}-{\bf X}{\bf b}\|_{2}^2 \Big] \prod_{j=1}^p \exp(-\frac
{\eta
_j}{\sigma} |b_j|)\vadjust{\eject}
\]
and the proof of Theorem~\ref{thm:poster} (see Appendix~6), we have that
the conditional distribution $p({\bf b}| \etab, \sigma, {\bf y})$
is proper. However,
the absolute terms $|b_j|$ make the form of $p({\bf b}| \etab, \sigma
, {\bf y}
)$ unfamiliar. Thus, a Gibbs sampling algorithm
is not readily available and we resort to an EM algorithm to estimate
the model.

\subsection{The ECME Estimation Algorithm}

Notice that if $p^{*}(b_j|{t_j}, \sigma)$ is proper, the corresponding
normalizing constant is given by
\[
2 \int_{0}^{\infty} \sigma^{-1} \exp\Big[-t_j \Psi\Big(\frac
{|b_j|} {\sigma} \Big)\Big] d |b_j|= 2 \int_{0}^{\infty} \exp\Big
[-t_j \Psi\Big(\frac{|b_j|} {\sigma} \Big) \Big] d (|b_j|/\sigma),
\]
which is independent of $\sigma$. Also, the conditional distribution
$p(\eta_j|b_j, t_j, \sigma)$ is independent of the normalizing term.
Specifically,
we always have that
\[
p(\eta_j|b_j, t_j, \sigma) = \frac{\exp\big(-\frac{\eta
_j}{\sigma} |b_j| \big)f_{T(t_j)}(\eta_j)} {\exp(-t_j \Psi(|b_j|
/\sigma))},
\]
which is proper.

As shown in Table~\ref{tab:exam}, except for LOG with $t>\xi$ which
can be transformed into a proper prior, the remaining Bernstein
functions cannot be transformed into proper priors.
In any case, our posterior computation is directly based on the
marginal pseudo-prior $p^{*}({\bf b}|{\bf t}, \sigma)$.
We ignore the involved normalizing term, because it is infinite if
$p^{*}({\bf b}|{\bf t}, \sigma)$ is improper and
it is independent of $\sigma$ if $p^{*}({\bf b}|{\bf t}, \sigma)$ is proper.

Given the $k$th estimates $({\bf b}^{(k)}, \sigma^{(k)})$ of $({\bf
b}, \sigma
)$ in the E-step of the EM algorithm, we compute
\begin{align*}
Q({\bf b}, \sigma|{\bf b}^{(k)}, \sigma^{(k)})
& \triangleq\log p({\bf y}|{\bf b}, \sigma) + \sum_{j=1}^p
\int{\log p[b_j | \eta_j, \sigma] p(\eta_j|b_j^{(k)}, \sigma
^{(k)}, t_j)} d \eta_j + \log p(\sigma) \\
& \propto-\frac{n+\alpha_{\sigma}}{2} \log\sigma{-} \frac{\|{\bf y}
{-}{\bf X}{\bf b}\|_2^2 + \beta_{\sigma}}{2 \sigma} - (p+1) \log
\sigma
\\
& \quad- \frac{1}{ \sigma} \sum_{j=1}^p
|b_j| \int\eta_j p(\eta_j|b_j^{(k)}, \sigma^{(k)}, t_j) d \eta_j.
\end{align*}
Here we omit some terms that are independent of
parameters $\sigma$ and ${\bf b}$. In fact, we only need to compute
$\EB
(\eta_j|b_j^{(k)}, \sigma^{(k)})$ in the E-step.
Considering that
\[
\int_{0}^{\infty}{\exp\big(-\frac{\eta_j}{\sigma} |b_j| \big)
f_{T(t_j)}(\eta_j) d \eta_j = \exp(-t_j \Psi(|b_j|/\sigma))},
\]
and taking the derivative w.r.t.\ $|b_j|$ on both sides of the above equation,
we have that
\[
w_j^{(k{+}1)} \triangleq\EB(\eta_j|b_j^{(k)}, \sigma^{(k)}, t_j) =
t_j \Psi'(|b_j^{(k)}|/\sigma^{(k)}).
\]

The M-step maximizes $Q({\bf b}, \sigma|{\bf b}^{(k)}, \sigma^{(k)})$
w.r.t.\
$({\bf b}, \sigma)$.
In particular, it is obtained that:
\begin{align*}
{\bf b}^{(k{+}1)} & = \argmin_{{\bf b}} \; \frac{1}{2} \| {\bf
y}{-}{\bf X}
{\bf b}\|_2^2
+ \sum_{j=1}^p w_j^{(k{+}1)} |b_j|, \\
\sigma^{(k{+}1)} & = \frac{1}{n {+} \alpha_{\sigma} {+} 2 p {+}2}
\Big\{ \|{\bf y}{-}{\bf X}{\bf b}^{(k{+}1)}\|_2^2 + \beta_{\sigma} +
2 \sum
_{j=1}^p w_j^{(k{+}1)} |b_j^{(k{+}1)}| \Big\}.
\end{align*}

The above EM algorithm is related to the linear local approximation
(LLA) procedure~\citep{ZouLi:2008}. Moreover, it shares the same
convergence property given in \citet{ZouLi:2008}
and \citet{ZhangEPGIG:2012}.

Subordinators help us to establish a direct connection between the
local regularization parameters $t_j$'s
and the latent shrinkage parameters $\eta_j$'s (or $T(t_j)$). However,
when we implement the MAP estimation, it is challenging how to select
these local regularization parameters.
We employ
an ECME (for ``Expectation/Conditional Maximization Either")
algorithm~\citep{LiuBio:1994,PolsonBS:2010} for learning about the
$b_j$'s and $t_j$'s simultaneously.
For this purpose, we suggest assigning $t_j$ Gamma prior $\Ga(\alpha
_{t}, 1/\beta_{t})$, namely,
\[
p(t_j) {=} \frac{\beta_{t}^{\alpha_{t}}}{\Gamma(\alpha_{t})}
{t_j^{\alpha_{t}-1}} \exp(-\beta_{t} t_j),
\]
because the full conditional distribution is also Gamma and given by
\[
[t_j|b_j, \sigma] \sim\Ga\big(\alpha_{t}, 1/[\Psi(|b_j|/\sigma)
+ \beta_{t}]\big).
\]
Recall that we here compute the full conditional distribution directly
using the marginal pseudo-prior $p^{*}(b_j|{t_j}, \sigma)$, because
our used Bernstein functions
in Table~\ref{tab:exam} cannot induce proper priors.
However, if $p^{*}(b_j|{t_j}, \sigma)$ is proper, the corresponding
normalizing term would rely on $t_j$. As a result, the full conditional
distribution of $t_j$ is possibly no longer Gamma or even not
analytically available.

Figure~\ref{fig:graphal0}-(a) depicts the hierarchical model for the
Bayesian penalized linear regression,
and Table~\ref{tab:alg} gives the ECME procedure where the E-step and
CM-step are respectively identical to the E-step and the M-step of the
EM algorithm,
with $t_j=t_j^{(k)}$. The CME-step updates the $t_j$'s with
\[
t_j^{(k{+}1)} = \frac{\alpha_{t} -1}{\Psi(|b^{(k)}_j|/\sigma^{(k)})
+ \beta_{t}}, \; j=1, \ldots, p.
\]
In order to make sure that $t_j>0$, it is necessary to assume that
$\alpha_{t}>1$. In the following experiments, we set $\alpha_{t}=10$.

We conduct experiments with the prior
$p({\bf b}_j) \varpropto t_j \sigma^{-1/2} \exp(-t_j (|b_j|/\sigma
)^{1/2})$ for comparison.
This prior is induced from the $\ell_{1/2}$-norm penalty, so
it is a proper specification. Moreover,
the full conditional distribution of $t_j$ w.r.t.\vadjust{\goodbreak}  its Gamma prior
$\Ga({\alpha_t}, 1/{\beta_t})$ is still Gamma;
that is,
\[
[t_j|b_j, \sigma] \sim\Ga\Big({\alpha_t}{+}2, \; 1/({\beta_t} {+}
\sqrt{|b_j|/\sigma})\Big).
\]
Thus, the CME-step for updating the $t_j$'s is given by
\begin{equation} \label{eqn:tthalf}
t_j^{(k{+}1)} = \frac{{\alpha_t} +1}{\sqrt{|b^{(k)}_j|/\sigma
^{(k)}} + {\beta_t}}, \; j=1, \ldots, p.
\end{equation}

The convergence analysis of the ECME algorithm was presented by \citet
{LiuBio:1994}, who proved that the ECME algorithm retains the
monotonicity property from the standard EM. Moreover, the ECME
algorithm based on pseudo-priors was
also used by \citet{PolsonScottSVM:2011}.

\begin{table}[!ht]
\begin{center}
\caption{ The Basic Procedure of the ECME Algorithm}
\label{tab:alg}
\begin{tabular}{|ll|}
\hline
{\bf E-step} & Identical to the E-step of the EM with $t_j=t_j^{(k)}$.
\\
{\bf CM-step} & Identical to the M-step of the EM with $t_j=t_j^{(k)}$.
\\
{\bf CME-step} & Set
$t_j^{(k{+}1)} = \frac{\alpha_{t} -1}{\Psi(|b^{(k)}_j|/\sigma
^{(k)}) + \beta_{t}}$. \\
\hline
\end{tabular}
\end{center}
\end{table}

As we have seen, $p({\bf t}|{\bf y}, {\bf b}, \sigma) = p({\bf t}|{\bf b},
\sigma
) = \prod_{j=1}^p p(t_j|b_j, \sigma)$
and $p(\sigma|{\bf b}, \etab, {\bf t}, {\bf y})$ are proper.
In the following theorem, we show that $p({\bf b}| \sigma, {\bf t},
{\bf y})$
and $p({\bf b}, \sigma, {\bf t}| {\bf y})$
are also proper. Moreover, when the improper prior $p(\sigma)
\varpropto\frac{1}{\sigma}$ (i.e., $\alpha_{\sigma}=\beta_{\sigma
}=0$ in the inverse Gamma prior) is used, Theorem~\ref{thm:poster}
shows that $p({\bf b}, \sigma, {\bf t}| {\bf y})$ is proper under certain
conditions.

\begin{theorem} \label{thm:poster} With the previous prior
specifications for ${\bf b}$, $\sigma$ and ${\bf t}$, we have that
$p({\bf b}|
\sigma, {\bf t}, {\bf y})$, $p({\bf b}, \sigma|{\bf t}, {\bf y})$ and
$p({\bf b},
\sigma, {\bf t}| {\bf y})$ are
proper. Suppose we use the improper prior $p(\sigma) \varpropto\frac
{1}{\sigma}$ for $\sigma$.
If ${\bf y}\notin\mathrm{range}({\bf X})$ \emph{(}the subspace
spanned by
the columns of ${\bf X}$\emph{)}, $p({\bf b}, \sigma|{\bf t}, {\bf
y})$ and
$p({\bf b}
, \sigma, {\bf t}| {\bf y})$ are proper.
\end{theorem}
The proof of Theorem~\ref{thm:poster} is given in Appendix~6. Notice that
the proof only requires that $\Psi(s)\geq0$, and does not involve the
other properties of the Bernstein function. In other words,
Theorem~\ref{thm:poster} is still held for any nonnegative but not
necessarily Bernstein function $\Psi$.
Theorem~\ref{thm:poster} shows that our ECME algorithm is to find the
MAP estimates of the parameters ${\bf b}$ and $\sigma$
as well as the MAP estimates of the local regularization parameters ${t_j}$'s.

In the EM algorithm of \citet{PolsonScott:2011}, the authors set
$t_1=\cdots=t_p\triangleq\nu$ as a global regularization parameter
and assumed it to be prespecified \citep[see Section 5.3 of
][]{PolsonScott:2011}.
This in fact leads to a parametric setting for the latent shrinkage
parameters $\eta$~\citep
{ZouLi:2008,CevherNIPS:2009,GarriguesfNIPS:2010,LeeCaronDoucetHolmes:2010,ArmaganDunsonLee}.
However, \citet{PolsonScott:2011} aimed to
construct sparse priors using increments of subordinators.
It is worth noting that\vadjust{\eject} \citet{CaronDoucet:icml}
regarded their model
as a nonparametric
model w.r.t.\ the regression coefficients $b$; that is, they treated
$b$ as a stochastic process of $T$. Thus,
the treatment of \citet{CaronDoucet:icml} is also different from ours.

\begin{figure*}[!ht]
\includegraphics{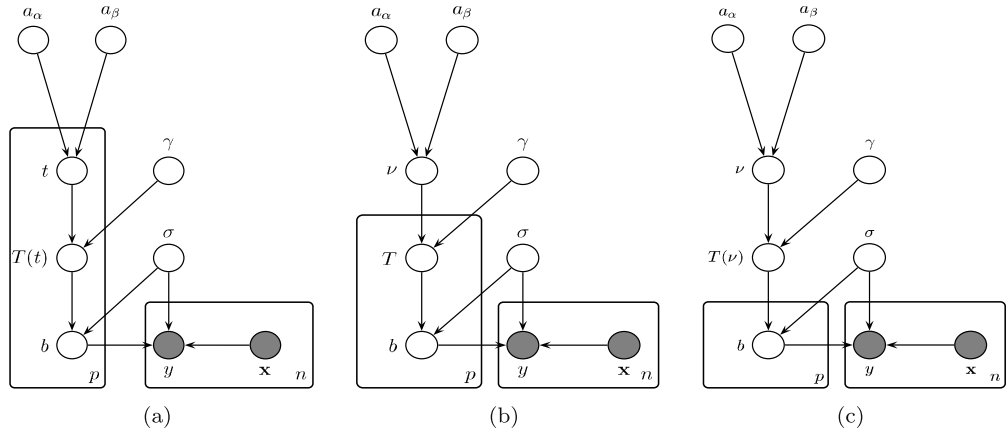}
\caption{Graphical representations for hierarchical regression models.
(a) Nonparametric setting for $T$, i.e.,
different $T$ have different $t$; (b)
Parametric setting for $T$, i.e., different $T$ share a common $\nu$;
(c) nonseparable setting, i.e., one $T$.}
\label{fig:graphal0}
\end{figure*}

\section{Experimental Analysis} \label{sec:experiment}

We now conduct empirical analysis with our ECME procedure described in
Algorithm~1 based on Figure~\ref{fig:graphal0}-(a).
We also implement the setting in \citet{PolsonScott:2011}, i.e.,
$t_1=\cdots=t_p\triangleq\nu$ and the $T_j$'s are independent given
$\nu$.
\citet{PolsonScott:2011} suggested that $\nu$ is prespecified as the
global regularization parameter.
In fact, we can also estimate $\nu$ under the ECME algorithm.
This setting is described in Figure~\ref{fig:graphal0}-(b) and the
corresponding ECME algorithm is given in Algorithm~2.

Notice that in the setting $t_1=\cdots=t_p\triangleq\nu$,
if the latent shrinkage $T(t)$ is treated as a stochastic process of
$t$, then the $b_j$'s share a common
$T(\nu)$. In this case, the marginal pseudo-prior for ${\bf b}$ is
nonseparable; that is, $p({\bf b})\propto\exp(-\frac{\nu}{\xi}
\Psi(\|
{\bf b}\|_1/\sigma))$.
Figure~\ref{fig:graphal0}-(c) illustrates the resulting model and the
corresponding ECME algorithm (see Algorithm~3) is also performed.

\begin{table}[!ht]
\begin{center}
\begin{small}
\begin{tabular}{|ll|}
\hline
\multicolumn{2}{|c|}{\bf Algorithm~1: ECME for Bayesian Regression
with Penalty $\Psi_{\rho}(|b|)$ or $\Phi_{\rho}(|b|)$} \\
& \\
{\bf E-step} & Given the current estimates ${\bf b}^{(k)}$ and
$t_j=t_j^{(k)}$, compute \\
& $w_j^{(k)} = {t_j^{(k)}} \Psi'_{\rho}( |b_j^{(k)}|/\sigma^{(k)}) $
\; or \; $w_j^{(k)} =
{t_j^{(k)}} \Phi'_{\rho}( |b_j^{(k)}|/\sigma^{(k)}) $, $j=1, \ldots
, p$ \\
{\bf CM-step} & Solve the following problem: \\
&
${\bf b}^{(k{+}1)} = \argmin_{{\bf b}} \; \frac{1}{2} \| {\bf
y}{-}{\bf X}
{\bf b}\|_2^2 +
\sum_{j=1}^p w_j^{(k{+}1)} |b_j|$, \\
& $\sigma^{(k{+}1)} = \frac{1}{\alpha_{\sigma} {+}n {+} 2 p {+}2}
\Big\{\beta_{\sigma} + \|{\bf y}{-}{\bf X}{\bf b}^{(k{+}1)}\|_2^2 +
2 \sum
_{j=1}^p w_j^{(k{+}1)} |b_j^{(k{+}1)}| \Big\}.$ \\
{\bf CME-step} & Compute \\
& $t_j^{(k{+}1)} = \frac{{\alpha_t} -1}{{\beta_t} + \Psi_{\rho}(
|b_j^{(k)}|/\sigma^{(k)}) }$
\; or \;
$t_j^{(k{+}1)} = \frac{{\alpha_t} -1}{{\beta_t} + \Phi_{\rho}(
|b_j^{(k)}|/\sigma^{(k)}) }$. \\
\hline
\end{tabular}
\end{small}
\end{center}
\end{table}

\begin{table}[!ht]
\begin{center}
\begin{small}
\begin{tabular}{|ll|}
\hline
\multicolumn{2}{|c|}{\bf Algorithm~2: ECME for Bayesian Regression
with Penalty $\Psi_{\rho}(|b|)$ or $\Phi_{\rho}(|b|)$} \\
& \\
{\bf E-step} & Given the current estimates ${\bf b}^{(k)}$ and $\nu
=\nu
^{(k)}$, compute \\
& $w_j^{(k)} = { \nu^{(k)}} {\Psi'_{\rho}( |b_j^{(k)}|/\sigma
^{(k)}) }$ \; or \; $w_j^{(k)} =
\nu^{(k)} { \Phi'_{\rho}( |b_j^{(k)}|/\sigma^{(k)}) }$, $j=1,
\ldots, p$ \\
{\bf CM-step} & Solve the following problem: \\
& ${\bf b}^{(k{+}1)} = \argmin_{{\bf b}} \; \frac{1}{2} \| {\bf
y}{-}{\bf X}
{\bf b}\|_2^2
+ \sum_{j=1}^p w_j^{(k{+}1)} |b_j|$, \\
& $\sigma^{(k{+}1)} = \frac{1}{\alpha_{\sigma} {+}n {+} 2 p {+}2}
\Big\{\beta_{\sigma} + \|{\bf y}{-}{\bf X}{\bf b}^{(k{+}1)}\|_2^2 +
2 \sum
_{j=1}^p w_j^{(k{+}1)} |b_j^{(k{+}1)}| \Big\}.$ \\
{\bf CME-step} & Compute \\
& $\nu^{(k{+}1)} = \frac{{\alpha_t} -1}{{\beta_t} + \sum_{j=1}^p
{\Psi_{\rho}( |b_j^{(k)}|/\sigma^{(k)}) }}$
\; or \;
$\nu^{(k{+}1)} = \frac{{\alpha_t} -1}{{\beta_t} + \sum_{j=1}^p {
\Phi_{\rho}( |b_j^{(k)}|/\sigma^{(k)}) }}$. \\
\hline
\end{tabular}
\end{small}
\end{center}
\end{table}

\begin{table}[!ht] 
\begin{center}
\begin{small}
\begin{tabular}{|ll|}
\hline
\multicolumn{2}{|c@{\ }|}{\bf Algorithm~3: ECME for Bayesian Regression
with Penalty $\Psi_{\rho}(\|{\bf b}\|_1)$ or $\Phi_{\rho}(\|{\bf
b}\|_1)$}
\\
& \\
{\bf E-step} & Given the current estimates ${\bf b}^{(k)}$ and $\nu
=\nu
^{(k)}$, compute \\
& $w^{(k)} = { \nu^{(k)}} {\Psi'(\|{\bf b}^{(k)}\|_1/\sigma^{(k)})}
$ \;
or \; $w^{(k)} =
\nu^{(k)} {\Phi'(\|{\bf b}^{(k)}\|_1/\sigma^{(k)})} $ \\
{\bf CM-step} & Solve the following problem: \\
& ${\bf b}^{(k{+}1)} = \argmin_{{\bf b}} \; \frac{1}{2} \| {\bf
y}{-}{\bf X}
{\bf b}\|_2^2
+ w^{(k{+}1)} \|{\bf b}\|_1$, \\
& $\sigma^{(k{+}1)} = \frac{1}{\alpha_{\sigma} {+} n {+} 2p{+}2}
\Big\{ \beta_{\sigma} + \|{\bf y}{-}{\bf X}{\bf b}^{(k{+}1)}\|_2^2 + 2
w^{(k{+}1)} \|{\bf b}^{(k{+}1)}\|_1 \Big\}.$ \\
{\bf CME-step} & Compute \\
& $\nu^{(k+1)} = \frac{{\alpha_t} -1}{{\beta_t} + {\Psi(\|{\bf b}
^{(k)}\|_1/\sigma^{(k)})}} $
\; or \;
$\nu^{(k{+}1)} = \frac{{\alpha_t} -1}{{\beta_t} + {\Psi(\|{\bf b}
^{(k)}\|_1/\sigma^{(k)})} }$. \\
\hline
\end{tabular}
\end{small}
\end{center}
\end{table}

We refer to the algorithms based on Figures~\ref{fig:graphal0}-(a),
(b) and (c) as ``Alg 1," ``Alg 2" and ``Alg 3,"
respectively.
We consider the nonconvex $\ell_{1/2}$, LOG, EXP, LFR and CEL
penalties to respectively implement these three algorithms.
The CME-step with the $\ell_{1/2}$-norm is based on expression (\ref
{eqn:tthalf}).
According to Theorem~\ref{thm:limit},
we can set, for instance, $\xi=\frac{\gamma}{1+\gamma}$ in LFR.
However, Theorem~\ref{thm:limit} also shows that when $\xi=\gamma$,
the two settings have the same asymptotic properties
as $\gamma\to0$. That is, the resulting model approaches the lasso.
We thus set $\xi=\gamma$ in ``Alg 1," and
$\xi=p \gamma$ in both ``Alg 2" and ``Alg 3." The settings are
empirically validated to be effective.
As we have mentioned, $\gamma$ is a global shrinkage parameter, so we
call it the global tuning parameter. In the experiments,
$\gamma$ and $\beta_t$ are selected via cross validation. As
hyperparameters $\alpha_{\sigma}$, $\beta_{\sigma}$, and ${\alpha
_t}$, we simply set $\alpha_{\sigma}=\beta_{\sigma}=0$, ${\alpha_t}=10$.

\begin{table*}[!ht]\setlength{\tabcolsep}{1.7pt}
\vskip-0.05in
\caption{Results of the three algorithms with $\ell_{1/2}$, LOG, EXP,
LFR and CEL on the simulated data sets. Here a standardized prediction
error (SPE) is used to evaluate the model prediction ability, and the
minimal achievable value for SPE is $1$.
And ``\checkmark'' denotes the proportion of correctly predicted zero
entries in ${\bf b}$, that is,
$\frac{\#\{i|b_i = 0 \, \textrm{and}\, \hat{b}_i = 0\}}{\#\{i|b_i =
0 \}}$;
if all the nonzero entries are correctly predicted, this score should
be $100\%$.
\label{tab:toy2}}\vspace*{6pt}
\begin{center}
\begin{tabular}{l | c c | c c | c c }
\hline
& SPE($\pm$STD) & \checkmark(\%) & SPE($\pm$STD) & \checkmark(\%) &
SPE($\pm$STD) & \checkmark(\%) \\
\hline\hline
& \multicolumn{2}{c|}{Data S}
& \multicolumn{2}{c|}{Data M}
& \multicolumn{2}{c}{Data L}\\ \hline
Alg 1+LOG & 1.0914($\pm$0.1703) & 98.24 & 1.1526($\pm$0.1025) & 97.42
& 1.4637($\pm$0.1735) & 90.04 \\
Alg 2+LOG & 1.1508($\pm$0.1576) & 85.25 & 1.3035($\pm$0.1821) & 87.35
& 1.5084($\pm$0.1676) & 88.67 \\
Alg 3+LOG & 1.1268($\pm$0.1754) & 86.33 & 1.5524($\pm$0.1437) & 91.21
& 1.5273($\pm$0.1567) & 85.25 \\
\hline
Alg 1+EXP & 1.1106($\pm$0.1287) & 98.67 & 1.1587($\pm$0.1527) & 97.98
& 1.4608($\pm$0.1557) & 87.55 \\
Alg 2+EXP & 1.1654($\pm$0.1845) & 87.36 & 1.3134($\pm$0.1152) & 88.45
& 1.5586($\pm$0.1802) & 85.34 \\
Alg 3+EXP & 1.1552($\pm$0.1495) & 80.33 & 1.5047($\pm$0.1376) & 93.67
& 1.5145($\pm$0.1594) & 84.56 \\
\hline
Alg 1+LFR & 1.0985($\pm$0.1824) & 98.67 & 1.1603($\pm$0.1158) & 98.34
& 1.4536($\pm$0.1697) & 89.23 \\
Alg 2+LFR & 1.1326($\pm$0.1276) & 86.35 & 1.3089($\pm$0.1367) & 87.28
& 1.5183($\pm$0.1507) & 85.67 \\
Alg 3+LFR & 1.1723($\pm$0.1534) & 84.28 & 1.3972($\pm$0.2356) & 88.33
& 1.5962($\pm$0.1467) & 86.53 \\
\hline
Alg 1+CEL & 1.1238($\pm$0.1145) & 96.12 & 1.1642($\pm$0.1236) & 98.26
& 1.4633($\pm$0.1346) & 89.58 \\
Alg 2+CEL & 1.1784($\pm$0.1093) & 84.67 & 1.4059($\pm$0.1736) & 89.67
& 1.5903($\pm$0.1785) & 85.23 \\
Alg 3+CEL & 1.1325($\pm$0.1282) & 85.23 & 1.3762($\pm$0.1475) & 90.32
& 1.5751($\pm$0.1538) & 82.65 \\
\hline
Alg 1+$\ell_{{1}/{2}}$ & 1.2436($\pm$0.1458) & 89.55 & 1.2937($\pm
$0.2033) & 94.83 & 1.5032($\pm$0.1633) & 85.86 \\
Alg 2+$\ell_{{1}/{2}}$ & 1.2591($\pm$0.1961) & 79.88 & 1.5902($\pm
$0.2207) & 83.50 & 1.6859($\pm$0.1824) & 83.58 \\
Alg 3+$\ell_{{1}/{2}}$ & 1.2395($\pm$0.2045) & 75.34 & 1.5630($\pm
$0.1642) & 80.83 & 1.6732($\pm$0.1711) & 80.67 \\
\hline
Lasso & 1.3454($\pm$0.3098) & 60.17 & 1.6708($\pm$0.2149) & 66.08 &
1.6839($\pm$0.1825) & 71.33 \\
\hline\hline
\end{tabular}
\end{center}
\end{table*}

Our analysis is based on a set of
simulated data, which are generated according to \citet
{MazumderSparsenet:11}. In particular,
we consider the following three data models --- ``small," ``medium" and
``large."
\begin{description}
\item[Data {S}:] $n = 35$, $p = 30$, ${\bf b}^{S} = (0.03, 0.07, 0.1, 0.9,
0.93, 0.97, {\bf0})^T$,
and $\Si^{S}$ is a $p\times p$ matrix with $1$ on the diagonal and
$0.4$ on the off-diagonal.\vadjust{\eject}
\item[Data {M}:] $n = 100$, $p = 200$,
${\bf b}^{M}$ has $10$ non-zeros such that $b^{M}_{20i+1}=1$ and $i=0, 1,
\cdots, 9$,
and $\Si^{M} = \{0.7^{|i-j|}\}_{1\leq i,j \leq p}$.
\item[Data {L}:] $n = 500$, $p = 1000$,
${\bf b}^{L} = ({\bf b}^{M}, \cdots, {\bf b}^{M})$,
and $\Si^{L} = \mathrm{diag}(\Si^{M}, \cdots, \Si^{M})$ (five blocks).
\end{description}
For each data model, we generate $n{\times} p$ data matrices ${\bf X}$
such that
each row of ${\bf X}$ is generated from a multivariate Gaussian
distribution with mean ${\bf0}_p$ and
covariance matrix $\Si^{S}$, $\Si^{M}$, or $\Si^{L}$.

We assume a linear model ${\bf y}= {\bf X}{\bf b}+ \epsi$ with multivariate
Gaussian predictors ${\bf X}$ and Gaussian errors.
We choose $\sigma$ such that the Signal-to-Noise Ratio (SNR) is a
specified value. Following the setting in \citet
{MazumderSparsenet:11}, we use $\mathrm{SNR} = 3.0$ in all the experiments.
We employ a standardized prediction error (SPE) to evaluate the model
prediction ability. The minimal achievable value for SPE is $1$.
Variable selection accuracy is measured by the correctly predicted
zeros and incorrectly predicted zeros in $\hat{{\bf b}}$.
The\vadjust{\eject} SNR and SPE are defined as
\[
\textrm{SNR} = \frac{\sqrt{{\bf b}^T \Si{\bf b}}}{\sigma} \quad
\mbox
{and} \quad
\textrm{SPE} = \frac{\EB({\bf y}- {\bf x}\hat{{\bf b}})^2}{\sigma^2}.
\]

For each data model, we generate training data of size $n$,
very large validation data and test data, each of size $10000$.
For each algorithm, the optimal global tuning parameters are chosen by
cross validation based on minimizing the average prediction errors.
With the model $\hat{{\bf b}}$ computed on the training data, we compute
SPE on the test data.
This procedure is repeated $100$ times, and we report the average and
standard deviation of SPE and the average of zero-nonzero error.
We use ``\checkmark'' to denote the proportion of correctly predicted
zero entries in ${\bf b}$, that is,
$\frac{\#\{i|b_i = 0 \, \textrm{and}\, \hat{b}_i = 0\}}{\#\{i|b_i =
0 \}}$;
if all the nonzero entries are correctly predicted, this score should
be $100\%$.

We report the results in Table~\ref{tab:toy2}. It is seen that our
setting in Figure~\ref{fig:graphal0}-(a) is better than the other two
settings in Figures~\ref{fig:graphal0}-(b) and (c) in both model
prediction accuracy and variable selection ability.
Especially, when the size of the dataset takes large values, the
prediction performance of the second setting becomes
worse. The several nonconvex penalties are competitive, but they
outperform the lasso. Moreover, we see that LOG, EXP, LFR and CEL
slightly outperform $\ell_{1/2}$.
The $\ell_{1/2}$ penalty indeed suffers from the problem of numerical
instability during
the EM computations.
As we know, the priors induced from LFR, CEL and EXP as well as LOG
with $t\leq\xi$ are improper,
but the prior induced from $\ell_{1/2}$ is proper.
The experimental results show that these improper priors work well,
even better than the proper case.

\begin{figure} 
\includegraphics{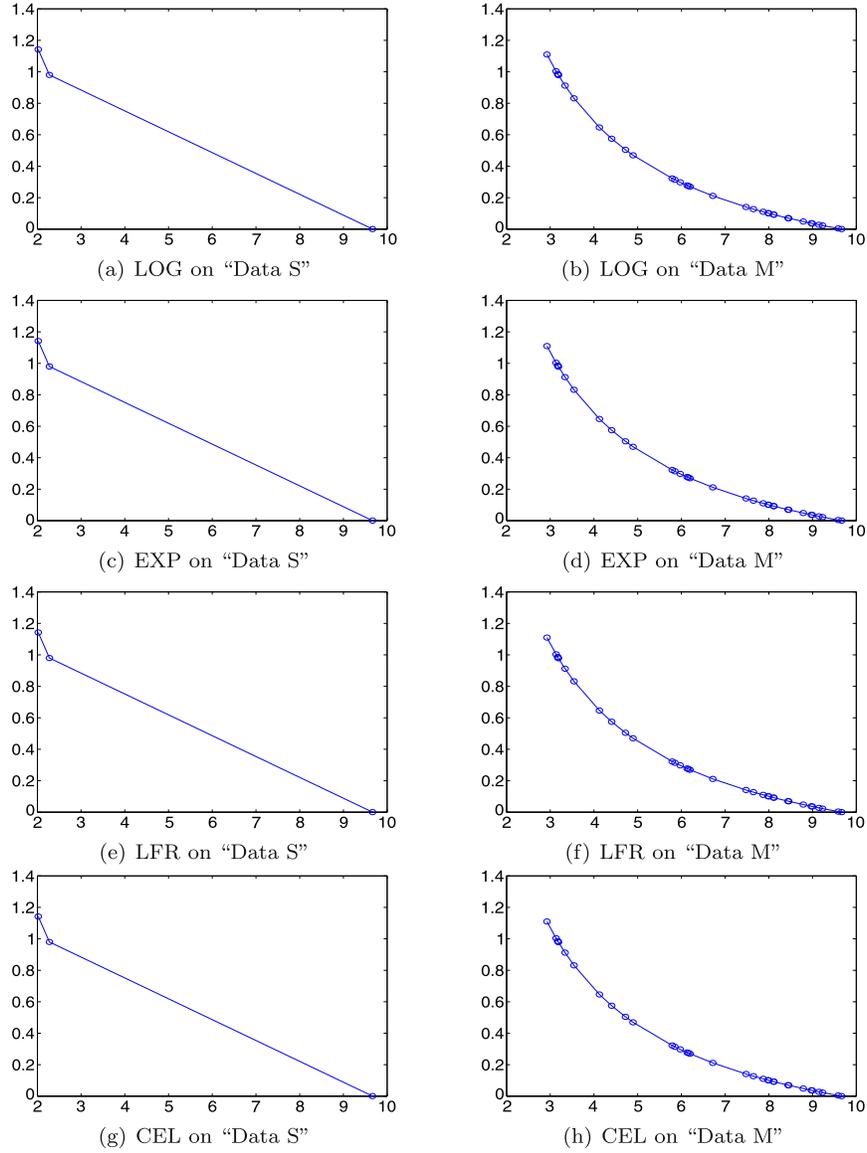}
\caption{The change of $|\hat{b}_{\pi_j}|$ vs. $\hat{t}_{\pi_j}$
on ``Data S" and ``Data M" where
$({\pi_1}, \ldots, {\pi_p})$ is the permutation of $({1}, \ldots,
{p})$ such that $\hat{t}_{\pi_1}\leq\cdots\leq\hat{t}_{\pi_p}$.}
\label{fig:tb1}
\end{figure}


Recall that in our approach each regression variable $b_j$ corresponds
to a distinct local tuning parameter $t_j$. Thus, it is interesting
to empirically investigate the inherent relationship between $b_j$ and
$t_j$. Let $\hat{t}_j$ be the estimate of $t_j$
obtained from our ECME algorithm (``Alg 1"), and $({\pi_1}, \ldots,
{\pi_p})$ be the permutation of
$({1}, \ldots, {p})$ such that $\hat{t}_{\pi_1}\leq\cdots\leq\hat
{t}_{\pi_p}$.
Figure~\ref{fig:tb1} depicts the change of $|\hat{b}_{\pi_j}|$ vs.\
$\hat{t}_{\pi_j}$ with LOG, EXP, LFR and CEL on ``Data {S}" and
``Data {M}."
We see that $|\hat{b}_{\pi_j}|$ is decreasing w.r.t.\ $\hat{t}_{\pi
_j}$. Moreover, $|\hat{b}_{\pi_j}|$ becomes 0 when
$\hat{t}_{\pi_j}$ takes some large value. A similar phenomenon is
also observed for ``Data {L}."
This thus shows that the subordinator is a powerful Bayesian
approach for variable selection.

\section{Conclusion} \label{sec:conclusion}

In this paper we have introduced subordinators into the definition of
nonconvex penalty functions.
This leads us to a Bayesian approach for constructing sparsity-inducing
pseudo-priors.
In particular, we have illustrated the use of two compound Poisson
subordinators: the compound Poisson Gamma subordinator and the negative
binomial subordinator.
In addition, we have established the relationship
between the two families of compound Poisson subordinators. That is,
we have proved that the two families of compound Poisson subordinators
share the same limiting behaviors.
Moreover, their densities at each time have the same mean and variance.

We have developed the
ECME algorithms for solving sparse learning problems based on the
nonconvex LOG, EXP, LFR and CEL penalties.
We have conducted the
experimental comparison with the state-of-the-art approach.
The results have shown that our nonconvex penalization approach is
potentially useful in high-dimensional
Bayesian modeling.
Our approach can be cast into
a point estimation framework. It is also interesting to fit a fully
Bayesian framework based on the MCMC estimation.
We would like to address this issue in future work.

\section*{Appendix 1: The L\'{e}vy Measure of the CEL Function}
\label{app:ly2}

Consider that
\begin{align*}
\log\Big[\frac{1{+}\rho}{\rho} -\frac{1}{\rho} \exp({-}\frac
{\rho}{1{+}\rho} \gamma s)\Big] &=
\log\Big[1-\frac{1}{1{+}\rho} \exp(-\frac{\rho}{1{+}\rho}
\gamma s)\Big] - \log\Big[1-\frac{1}{1{+}\rho}\Big] \\
& = \sum_{k=1}^{\infty} \frac{1}{k (1{+}\rho)^k} \Big[1- \exp\Big
({-}\frac{\rho}{1{+}\rho} k \gamma s\Big)\Big] \\
& = \sum_{k=1}^{\infty} \frac{1}{k (1{+}\rho)^k} \int_{0}^{\infty
}(1- \exp(- u s))
\delta_{\frac{\rho k \gamma}{1{+}\rho}}(u) d u.
\end{align*}
We thus have that $\nu(du) = \frac{1+\rho}{\xi} \sum_{k=1}^{\infty
} \frac{1}{k (1{+}\rho)^k} \delta_{\frac{\rho k \gamma}{1{+}\rho
}}(u) d u$.

\section*{Appendix 2: The Proof of Proposition~\ref{pro:8}}
\label{app:p8}

We here give an alternative proof of Proposition~\ref{pro:8}-(d),
which is immediately obtained from the following lemma.

\begin{lemma} Let $X$ take discrete value on ${\NB\cup\{0\}}$ and
follow negative binomial distribution $\mathrm{Nb}(r, p)$.
If $r$ converges to a positive constant as $p\to0$, $p X$ converges in
distribution to a Gamma random variable with shape $r$ and scale $1$.%
\end{lemma}
\begin{proof}
Since
\[
F_{p X}(x) = \Pr(p X\leq x) = \sum_{
\begin{array}{l} kp \leq x \\ k=0
\end{array}
}^{\infty} \frac{\Gamma(k+r)}{\Gamma(r) \Gamma(k{+}1)} p^r (1-p)^k,
\]
we have that
\begin{align*}
\lim_{p\to0+} F_{p X}(x) &= \lim_{p\to0+} \sum_{
\begin{array}{l} kp \leq x \\ k=0
\end{array}
}^{\infty} \frac{\Gamma(k+r)}{\Gamma(r) \Gamma(k{+}1)} p^r (1-p)^k
\\
& = \lim_{p\to0+} \sum_{
\begin{array}{l} kp \leq x \\ k=1
\end{array}
}^{\infty} \frac{\Gamma(\frac{k p}{p}+r)}{\Gamma(r) \Gamma(\frac
{k p}{p}{+}1)} p p^{r-1} (1-p)^k\\
&= \frac{1}{\Gamma(r)} \int_{0}^{x}{\lim_{p \to0} \frac{\Gamma
(\frac{u}{p}{+}r) }{\Gamma(\frac{u}{p} {+}1)}} p^{r-1} (1-p)^{u/p} du.
\end{align*}
Notice that $\liml_{p\to0}(1-p)^{u/p}= \exp(-u)$ and
\begin{align*}
\lim_{p\to0} \frac{\Gamma(\frac{u}{p}{+}r) } {\Gamma(\frac{u}{p}
{+}1)} p^{r-1} &= \lim_{p\to0} \frac{\Big(\frac{u}{p}+r\Big
)^{\frac{u}{p}+r-\frac{1}{2}} \exp(-\frac{u}{p}-r) } {\Big(\frac
{u}{p}+1\Big)^{\frac{u}{p}+1-\frac{1}{2}} \exp(-\frac{u}{p}-1) }
p^{r-1} \\
& =\exp(1-r) \lim_{p\to0} \left(\frac{\frac
{u}{p}{+}1{+}r{-}1}{\frac{u}{p}+1} \right)^{\frac{u}{p}+1} \left
(\frac{\frac{u}{p}{+}r}{\frac{u}{p}+1} \right)^{-\frac{1}{2}} \Big
(\frac{u}{p}{+}r \Big)^{r-1} p^{r-1} \\
& = \exp(1-r) \exp(r-1) u^{r-1} = u^{r-1}.
\end{align*}
This leads us to
\[
\lim_{p\to0+} F_{p X}(x) =\int_{0}^{x}{\frac{u^{r-1}}{\Gamma(r)}
\exp(-u) du}.\qedhere
\]
\end{proof}

Similarly, we have that
\begin{align*}
\lim_{\rho\to0} \nu(d u) &= \frac{\rho+1}{\xi} \sum
_{k=1}^{\infty}{ \frac{\rho}{k \rho(1{+}\rho)^{k\rho/\rho} }
\delta_{k \rho\gamma/(1{+}
\rho)} d u } \\
&= \frac{1}{\xi}\int_{0}^{\infty} z^{-1} \exp(-z) \delta_{z
\gamma}(u) d z \\
& = \frac{1}{\xi} u^{-1} \exp(-u/\gamma).
\end{align*}

\section*{Appendix 3: The Proof of Theorem~\ref{thm:poigam}}
\label{app:poigam}

\begin{proof}
Consider a mixture of $\Ga(\eta|k \nu, \beta)$ with $\Po(k|\lambda
)$ mixing. That is,
\begin{align*}
p(\eta) & = \sum_{k=0}^{\infty} \Ga(\eta|k \nu, \beta) \Po
(k|\lambda) \\
& = \sum_{k=0}^{\infty} \frac{\beta^{-k\nu}}{\Gamma(k \nu)} \eta
^{k \nu-1}
\exp(-\frac{\eta}{\beta}) \frac{\lambda^k}{k!} \exp(- \lambda)
\\
& = \lim_{k\to0} \frac{ \eta^{k\nu-1} \lambda^k} { \beta^{k \nu
}\Gamma(k \nu) k!}
\exp(-\frac{\eta}{\beta}) \exp(- \lambda) +
\sum_{k=1}^{\infty} \frac{{\lambda}^{k} (\eta/\beta)^{k\nu}
}{\eta\Gamma(k \nu) k!} \exp(-(\frac{\eta}{\beta} + {\lambda
})) \\
& = \exp(-{\lambda})\Big\{ \delta_{0}(\eta) + \exp(-\frac{\eta
}{\beta} ) \sum_{k=1}^{\infty} \frac{{\lambda}^{k} (\eta/\beta
)^{k\nu} }{\eta\Gamma(k \nu) k!} \Big\}.
\end{align*}
Letting $\lambda=\frac{\rho t}{\xi(\rho{-}1)}$, $\nu=\rho-1$ and
$\beta=\frac{\gamma}{\rho}$, we have that
\[
p(\eta) = \exp\Big( {-}\frac{\rho t}{\xi(\rho{-}1)} \Big) \left
\{\delta_{0}(\eta) + \exp\Big({-}\frac{\rho\eta}{\gamma} \Big)
\eta^{-1} \sum_{k=1}^{\infty}
\frac{(\rho t/\xi)^{k} (\rho\eta/\gamma)^{k (\rho{-}1)} } {k!
(\rho{-}1)^k \Gamma(k(\rho{-}1))} \right\}.
\]

We now consider a mixture of $\Po(k|\phi\lambda)$ with $\Ga(\lambda
|\psi, 1/\beta)$ mixing. That is,
\begin{align*}
\Pr(T(t) =k \alpha) & = \int_{0}^{\infty} \Po(k|\lambda\phi) \Ga
(\lambda|\psi, 1/\beta) d \lambda\\
& = \int_{0}^{\infty}{\frac{(\lambda\phi)^k}{k!} \exp(-\lambda
\phi) \frac{\beta^{\psi}}{\Gamma(\psi)} \lambda^{\psi-1} \exp
(-\beta\lambda) d \lambda} \\
& = \frac{\beta^{\psi}}{\Gamma(\psi)} \frac{\Gamma(\psi
{+}k)}{k!} \frac{\phi^{k}}{(\phi+\beta)^{k+\psi}},
\end{align*}
which is $\mathrm{Nb}(T(t)|\psi, \beta/(\beta+\phi))$. Let $\psi
=(\rho+1) t/\xi$, $\phi=1$, $\beta=\rho$ and $q= \frac{\beta
}{\phi{+}\beta}$. Thus,
\[
\Pr(T(t)=k \alpha) = \frac{\Gamma(k{+}(\rho{+}1) t/\xi)}{k!
\Gamma((\rho{+}1) t/\xi)} q^{(\rho{+}1) t/\xi} (1-q)^k.\qedhere
\]
\end{proof}

\section*{Appendix 4: The Proof of Theorem~\ref{thm:limit}}
\label{app:limit}

\begin{proof}
Since $\liml_{\gamma\to0} \frac{\rho+1}{\rho\gamma}\Big
[1-(1+\frac{\gamma}{\rho+1})^{-\rho} \Big]=1$, we only need to consider
the case that $\xi=\gamma$.
Recall that
$\PG(t/\xi, \gamma, \rho)$, whose mean and variance are
\[
\EB(T(t)) = \frac{\gamma t}{\xi} =t \quad\mbox{ and } \quad{\Var
}(T(t))= \frac{\gamma^2 t}{\xi} = \gamma t
\]
whenever $\xi=\gamma$.
By Chebyshev's inequality, we have that
\begin{equation*}
\Pr\{ |T(t) - t| \geq\epsilon\} \leq
\frac{\gamma t}{\epsilon^2}.
\end{equation*}
Hence, we have that
\[
\lim_{\gamma\to0} \Pr\{ |T(t) - t| \geq\epsilon\} = 0.
\]
Similarly, we have Part (b).
\end{proof}

\section*{Appendix 5: The Proof of Proposition in Expression (\ref
{eqn:relat})}
\label{app:ee}

\begin{proof} We first note that
\[
2 \exp(\gamma s) = 2 + 2 \gamma s + (\gamma s)^2 + \frac{2}{3}(\gamma
s)^3+ \cdots,
\]
which implies that $2 \exp(\gamma s)-1 -(\gamma s+1)^2> 0$ for $s>0$.
Subsequently, we have that $\frac{d}{ds} \big[ \log(2 {-} \exp
({-}\gamma s)) - \frac{\gamma s}{1+\gamma s} \big] \leq0$. As a
result, $\log(2- \exp(-\gamma s)) -\frac{\gamma s}{1+\gamma s}<0$
for $s>0$.
As for $\frac{\gamma s}{\gamma s {+}1} \leq1 {-} \exp( {-} \gamma
s)$, it is directly obtained from that
\[
\frac{\gamma s}{\gamma s {+}1} =1- \frac{1}{1+\gamma s} =1 - \exp
(-\log(1+\gamma s))\leq1- \exp(-\gamma s).
\]
Since $\frac{d}{ds} \big[ 1 {-} \exp( {-} \gamma s) - \log\big
({\gamma} s {+}1 \big)\big] = \frac{\gamma}{\exp(\gamma s)} -
\frac{\gamma}{1+\gamma s}<0$ for $s>0$, we have that $1 {-} \exp(
{-} \gamma s) - \log\big({\gamma} s {+}1 \big)<0$ for $s>0$.
\end{proof}

\section*{Appendix 6: The Proof of Theorem~\ref{thm:poster}}
\label{app:ff}

\begin{proof} First consider that
\[
p({\bf b}|\sigma, {\bf t}, {\bf y}) \varpropto\frac{1}{(2\pi\sigma
)^{\frac
{n}{2}}} \exp\big[-\frac{1}{2 \sigma} \|{\bf y}- {\bf X}{\bf b}\|
_2^2 \big]
\prod_{j=1}^p \sigma^{-1} \exp\Big(-t_j \Psi\Big(\frac{|b_j|}
{\sigma} \Big) \Big).
\]
To prove that $p({\bf b}|\sigma, {\bf t}, {\bf y})$ is proper, it
suffices to
obtain that
\[
\frac{1}{(2\pi\sigma)^{\frac{n}{2}}} \int{\exp\big[-\frac{1}{2
\sigma} \|{\bf y}- {\bf X}{\bf b}\|_2^2 \big] \prod_{j=1}^p \sigma^{-1}
\exp
\Big(-t_j \Psi\Big(\frac{|b_j|} {\sigma} \Big) \Big) d {\bf b}<
\infty}.
\]
It is directly computed that
\begin{align} \label{eqn:pf01}
& \exp\big[-\frac{1}{2 \sigma} \|{\bf y}- {\bf X}{\bf b}\|_2^2 \big]
\nonumber
\\
& = \exp\big[ {-} \frac{1}{2 \sigma} ({\bf b}{-} {\bf z})^T {\bf
X}^T {\bf X}
({\bf b}-
{\bf z}) \big] \times\exp\big[- \frac{1}{2 \sigma} {\bf y}^T
({\bf I}_n - {\bf X}
({\bf X}^T {\bf X})^{+} {\bf X}^T ) {\bf y}\big],
\end{align}
where ${\bf z}= ({\bf X}^T {\bf X})^{+} {\bf X}^T {\bf y}$ and $({\bf
X}^T {\bf X})^{+}$ is the
Moore-Penrose pseudo inverse of matrix ${\bf X}^T {\bf X}$~\citep
{Magnus:1999}. Here we use the well-established properties that ${\bf
X}({\bf X}
^T {\bf X})^{+} ({\bf X}^T {\bf X}) ={\bf X}$ and $({\bf X}^T {\bf
X})^{+} ({\bf X}^T {\bf X}) ({\bf X}^T
{\bf X})^{+} = ({\bf X}^T {\bf X})^{+}$.
Notice that if ${\bf X}^T {\bf X}$ is nonsingular, then $({\bf X}^T
{\bf X})^{+} = ({\bf X}
^T {\bf X})^{-1}$. In this case, we consider a conventional multivariate
normal distribution $N({\bf b}|{\bf z}, \sigma({\bf X}^T {\bf
X})^{-1})$. Otherwise,
we consider a singular multivariate
normal distribution $N({\bf b}|{\bf z}, \sigma({\bf X}^T {\bf X})^{+})$
~\citep{Mardia:1979}, the density of which is given by
\[
\frac{\prod_{j=1}^q \sqrt{\lambda_j({\bf X}^T {\bf X})} }{(2 \pi
\sigma
)^{q/2}} \exp\big[ {-} \frac{1}{2 \sigma} ({\bf b}{-} {\bf z})^T
{\bf X}^T
{\bf X}
({\bf b}- {\bf z}) \big].
\]
Here $q = \rk({\bf X})$, and $\lambda_j({\bf X}^T {\bf X})$, $j=1,
\ldots, q$,
are the positive eigenvalues of ${\bf X}^T {\bf X}$.
In any case, we always write $N({\bf b}|{\bf z}, \sigma({\bf X}^T {\bf
X})^{+})$.
Thus, $\int{ \exp\big[-\frac{1}{2 \sigma} \|{\bf y}- {\bf X}{\bf
b}\|_2^2
\big
] d{\bf b}< \infty}$. It then follows the propriety of $p({\bf
b}|\sigma,
{\bf t}, {\bf y})$ because
\[
\exp\big[{-}\frac{1}{2 \sigma} \|{\bf y}- {\bf X}{\bf b}\|_2^2 \big
] \prod
_{j=1}^p \exp\Big( {-} t_j \Psi\Big(\frac{|b_j|} {\sigma} \Big)
\Big)\leq\exp\big[ {-} \frac{1}{2 \sigma} \|{\bf y}- {\bf X}{\bf
b}\|_2^2
\big].
\]

We now consider that
\[
p({\bf b}, \sigma| {\bf t}, {\bf y}) \varpropto{\sigma^{-(\frac
{n+\alpha
_{\sigma}+2p}{2} +1)} } \exp\Big[-\frac{\|{\bf y}- {\bf X}{\bf b}\|
_2^2 {+}
\beta_{\sigma}}{2 \sigma} \Big] \prod_{j=1}^p \exp\Big(-t_j \Psi
\Big(\frac{|b_j|} {\sigma} \Big) \Big).
\]
Let $\nu= {\bf y}^T[{\bf I}_n- {\bf X}({\bf X}^T{\bf X})^{+} {\bf
X}^T] {\bf y}$. Since the matrix
${\bf I}_n- {\bf X}({\bf X}^T{\bf X})^{+} {\bf X}^T$ is positive
semidefinite, we obtain
$\nu\geq0$. Based on expression (\ref{eqn:pf01}), we can write
\[
{\sigma^{-(\frac{n{+}\alpha_{\sigma}{+}2p}{2} +1)} } \exp\Big
[{-}\frac{\|{\bf y}{-} {\bf X}{\bf b}\|_2^2 {+} \beta_{\sigma}}{2
\sigma}
\Big
] \varpropto N({\bf b}|{\bf z}, \sigma({\bf X}^T {\bf X})^{+})
{\IGa}(\sigma|\frac{\alpha_{\sigma} {+}n{+}2p{-}q}{2}, \nu{+}
\beta_{\sigma}).
\]
Subsequently, we have that
\[
\int{{\sigma^{-(\frac{n+\alpha_{\sigma}+2p}{2} +1)} } \exp\Big
[-\frac{\|{\bf y}- {\bf X}{\bf b}\|_2^2 + \beta_{\sigma}}{2 \sigma}
\Big]
d {\bf b}
d \sigma} < \infty,
\]
and hence,
\[
\int{{\sigma^{-(\frac{n+\alpha_{\sigma}+2p}{2} +1)} } \exp\Big
[-\frac{\|{\bf y}- {\bf X}{\bf b}\|_2^2 + \beta_{\sigma}}{2 \sigma}
\Big]
\prod_{j=1}^p \exp\Big(-t_j \Psi\Big(\frac{|b_j|} {\sigma} \Big
) \Big) d {\bf b}d \sigma} < \infty.
\]
Therefore $p({\bf b}, \sigma|{\bf t}, {\bf y})$ is proper.

Thirdly, we take
\begin{align*}
p({\bf b}, \sigma, {\bf t}| {\bf y}) & \varpropto\frac{\exp\Big
[-\frac
{\|
{\bf y}- {\bf X}{\bf b}\|_2^2 + \beta_{\sigma}}{2 \sigma} \Big]}
{\sigma
^{\frac{n+\alpha_{\sigma}+2p}{2} +1} } \prod_{j=1}^p \Big\{\exp
\Big({-}t_j \Psi\Big(\frac{|b_j|} {\sigma} \Big) \Big) \frac{
t_j^{{\alpha_t}{-} 1} \exp({-} {\beta_t} t_j)}{\Gamma({\alpha_t})}
\Big\} \\
& \triangleq F({\bf b}, \sigma, {\bf t}).
\end{align*}
In this case, we compute
\[
\int{F({\bf b}, \sigma, {\bf t}) d {\bf b}d \sigma d {\bf t}}= \int
{\frac
{\exp\Big[-\frac{\|{\bf y}- {\bf X}{\bf b}\|_2^2 + \beta_{\sigma}}{2
\sigma}
\Big]} {\sigma^{\frac{n+\alpha_{\sigma}+2p}{2} +1} } \prod
_{j=1}^p \frac{1}{ \Big({\beta_t} {+} \Psi\Big(\frac{|b_j|}
{\sigma} \Big) \Big)^{{\alpha_t}}} d {\bf b}d \sigma}.
\]
Similar to the previous proof, we also have that
\[
\int{F({\bf b}, \sigma, {\bf t}) d {\bf b}d \sigma d {\bf t}} <
\infty
\]
because ${ \Big({\beta_t} {+} \Psi\Big(\frac{|b_j|} {\sigma} \Big
) \Big)^{-{\alpha_t}}} \leq{{\beta_t}^{-{\alpha_t}}}$.
As a result, $p({\bf b}, \sigma, {\bf t}| {\bf y})$ is proper.

Finally, consider the setting that $p(\sigma) \varpropto\frac
{1}{\sigma}$. That is, $\alpha_{\sigma}=0$ and $\beta_{\sigma}=0$.
In this case,
if ${\bf y}\notin\mathrm{range}({\bf X})$, we obtain $\nu>0$ and
$q<n$. As a result,
we use the inverse Gamma distribution ${\IGa}(\sigma|\frac
{n{+}2p{-}q}{2}, \nu)$. Thus, the results still hold.
\end{proof}



%

%
\begin{acknowledgement}
The authors would like to thank the Editors and two anonymous referees
for their constructive comments and suggestions on the original version
of this paper.
The authors would especially like to thank the Associate Editor for
giving extremely detailed comments
on earlier drafts. This work has been supported in part by the Natural
Science Foundation of
China (No. 61070239).
\end{acknowledgement}

\end{document}